\documentclass{article}

\usepackage[nonatbib,preprint]{neurips_2025}

\usepackage[utf8]{inputenc}
\usepackage[T1]{fontenc}
\usepackage{hyperref}
\usepackage{url}
\usepackage{booktabs}
\usepackage{amsfonts}
\usepackage{nicefrac}
\usepackage{microtype}
\usepackage{xcolor}       

\usepackage{amsmath}
\usepackage{amssymb}
\usepackage{amsthm}
\usepackage{graphicx}
\usepackage{subcaption}
\usepackage{color}

\usepackage{thmtools, thm-restate}

\providecommand{\Rd}{\mathbb{R}^d}

\theoremstyle{plain}
\newtheorem{theorem}{Theorem}[section]

\newtheorem{corollary}[theorem]{Corollary}
\theoremstyle{definition}
\newtheorem{definition}[theorem]{Definition}

\theoremstyle{remark}
\newtheorem{remark}[theorem]{Remark}

\newtheorem{fact}{Fact}

\author{
    David~Silva-Sánchez\\
	Department of Applied Mathematics\\
	Yale University\\
	New Haven, CT 06511 \\
    \texttt{david.silva@yale.edu}\\
	\And
    Roy~R.~Lederman\\
    Department of Statistics and Data Science\\
    Yale University\\
    New Haven, CT, 06511\\
    \texttt{roy.lederman@yale.edu}\\
}

\begin{document}

\title{An Observation on Lloyd's k-Means Algorithm in High Dimensions}

\maketitle

\begin{abstract}
    Clustering and estimating cluster means are core problems in statistics and machine learning, with k-means and Expectation Maximization (EM) being two widely used algorithms. In this work, we provide a theoretical explanation for the failure of k-means in high-dimensional settings with high noise and limited sample sizes, using a simple Gaussian Mixture Model (GMM). We identify regimes where, with high probability, almost every partition of the data becomes a fixed point of the k-means algorithm. This study is motivated by challenges in the analysis of more complex cases, such as masked GMMs, and those arising from applications in Cryo-Electron Microscopy.
\end{abstract}

\section{Introduction}\label{sec:introduction}

The issues of clustering data samples and estimating cluster parameters are prevalent in many applications across statistics and data science. One of the most common formulations of this problem is k-means, which aims to minimize intra-cluster variance. Solutions to the k-means problem are typically approximated using Lloyd's k-means algorithm\cite{lloyd1982least}; we will follow  common practice and refer to it as {\emph{the k-means algorithm}} for brevity.
The formulation and algorithm are summarized in Section \ref{sec:preliminaries:kmeans}. Another common formulation seeks to find the maximum likelihood estimates of cluster parameters under the assumption of a Gaussian Mixture Model (GMM). This problem is often approximated using the Expectation Maximization (EM) algorithm; a detailed discussion of the EM algorithm is deferred to future work.

It has been observed that the k-means algorithm encounters difficulties in high-dimensional settings (e.g., \cite{hartigan1975clustering,steinley2006k,zha2001spectral,ding2004k}). This paper aims to elucidate one of the mechanisms that contribute to these practical difficulties in finite-sample high-dimensional settings. We demonstrate theoretically in a probabilistic generative model that in these settings, this iterative algorithm has numerous arbitrary fixed points where it can become trapped; under appropriate conditions, almost every partition of the finite data into clusters represents a fixed point of the algorithm.

Informally, the results point to a ``catastrophic failure'' of the algorithm in certain regimes, where the algorithm cannot make any progress from (almost) any initialization even in a simple two clusters problem; in the extreme case, the clusters produced directly from the initialization are the output of the algorithm.
Counterintuitively, we find that under certain conditions, adding information in the form of additional dimensions can make the algorithm perform {\emph{worse}}, although the problem becomes {\emph{easier}} from an information theoretical point of view and from a practical point of view (using simple dimensionality reduction techniques); this effect appears to be more pronounced when the initialization is not good. 
Although other works explore the Lloyd k-means algorithm in high-dimensional regimes (e.g., \cite{zha2001spectral,ding2004k,telgarsky2010hartigan}), we have failed to find a discussion addressing the aspects covered in this paper. For brevity, we limit the theoretical discussion in this preliminary work to a simple two-class ($k=2$) isotropic GMM and to a limited conservative analysis.
The GMM model in the context of this paper is summarized in Appendix \ref{sec:GMM}.

This work is motivated by applications in Cryo-Electron Microscopy (cryo-EM; the "EM" part of the acronym is unrelated to Expectation Maximization) \cite{singer2020computational,bendory2020single,scheres2012relion,toader2023methods,sorzano2022bias}.
Although the estimation and clustering problems in cryo-EM are more nuanced than the models discussed here, we hypothesize that the phenomena explored in this paper are also manifested in cryo-EM algorithms. We believe that the insights and tools developed in this and subsequent papers will aid in understanding and mitigating these issues in cryo-EM applications. 
Motivated by this application, we consider a relatively high-noise high-dimension regime in this paper; the relevance of this regime is discussed in Section \ref{sec:settings}. For brevity, we do not delve into the estimation problems specific to cryo-EM, deferring a detailed treatment of this area to future work. We introduce the masked-GMM model (Equation (\ref{eq:maskedGMM}) in Appendix \ref{sec:settings}) as a model with some of the properties of cryo-EM that is much simpler and more broadly applicable; the model provides motivation for some of the choices we made in this work and some idea as to future directions. 

A more detailed background and a concise overview of relevant facts and definitions are found in Appendix \ref{sec:preliminaries}.
 The theoretical component of the paper is presented in Section \ref{sec:analysis}. Numerical experiments are presented in Section \ref{sec:numerical}. A brief discussion is presented in Section \ref{sec:conclusions}.
 Most of the proofs are presented in Appendix \ref{sec:proofs}. Additional analysis is presented in Appendix \ref{sec:additional} and additional experimental results are presented in Appendix {\ref{sec:additional_experiments}}

\section{Analytical Apparatus}\label{sec:analysis}

\subsection{The Difference Between Scaled $\chi^2$ Variables}

A key question in this paper will be the comparison of the distance of a sample to two different cluster centers. 
We will replace these distances with scaled chi-squared random variables through a first-order stochastic dominance relation (see Definition \ref{def:FSD}). We refer to these variables as the ``worst case'' for reasons that become clear in the proof. 
In this section, we provide a bound on the probability that variables related to scaled chi-squared random variables through first-order stochastic dominance differ by a certain amount.

\begin{restatable}[]{lemma}{lemchibound}\label{lem:chibound}
        Let $Y_1 \geq 0$ be a non-negative random variable with first-order stochastic dominance (see Definition \ref{def:FSD}) over the scaled chi-squared distribution with $d$ degrees of freedom, and $Y_2 \geq 0$ be a non-negative random variable 
        first-order stochastically dominanated by the scaled chi-squared distribution with $d$ degrees of freedom:
        \begin{equation}\label{eq:chibound:Y1Y2}
            Y_1 \succeq_{FSD} b_1 Z_1 , ~~~~ Z_1 \sim \chi^2_d  ~~~,~~~ Y_2 \preceq_{FSD} b_2 Z_2, ~~~ Z_2 \sim \chi^2_d
        \end{equation}
        where $Z_1 \sim \chi^2_d$ and $Z_2 \sim \chi^2_d$ are chi-squared random variables with $d$ degrees of freedom.
        Let $b_1 > b_2 >0$ be positive constants, and $m \in \mathbb{R}$.
        $Y_1$ and $Y_2$ are not necessarily independent.
        Then,
        \begin{equation}\label{eq:chibound}
            Pr\left( Y_1 - Y_2  \leq m \right) \leq 
              \exp\left( m \frac{b_1-b_2}{8 b_1 b_2} \right) \left( \frac{\left(b_1+b_2\right){}^2}{4 b_1 b_2} \right)^{-d/4} 
        \end{equation}
\end{restatable}

The proof of Lemma \ref{lem:chibound} is provided in Appendix \ref{sec:proofs}.

\begin{remark}\label{rem:chibound:rho}
    We observe that $\left( \frac{\left(b_1+b_2\right){}^2}{4 b_1 b_2} \right) > 1$ since $b_1>b_2>0$.
\end{remark}

\begin{remark}[``Worst Case `` Intuition]
    $Y_1$ is larger than $b_1 Z_1$ (in the stochastic dominance sense), and $Y_2$ is smaller than $b_2 Z_2$ (in the stochastic dominance sense), 
    so, intuitively, $b_1 Z_1 - b_2 Z_2$ is a ``worse case'' than $Y_1 - Y_2$:
       $ Pr\left( Y_1 - Y_2  \leq m \right)  \leq Pr\left( b_1 Z_1 - b_2 Z_2  \leq m \right)$.
    However, this statement is not necessarily accurate without considering the dependence between variables. 
    Lemma \ref{lem:chibound} takes the possible dependence into consideration. 
\end{remark}

\subsection{The Distribution of Distances to the Cluster Centers}\label{sec:distances}

In this section, we analyze the distribution of distances from a sample to the cluster centers in the k-means algorithm.
For simplicity, we restrict the formal analysis in this paper to the behavior of the k-means algorithm in the case $K=2$ of two clusters. 

Let the cluster centers be i.i.d. $\mu^{\text{True}}_k \in \mathbb{R}^d \sim N\left(0, \tau^2 I_d \right)$.
Let the samples be generated from the cluster centers with independent noise: $x_i = \mu^{\text{True}}_{z^{\text{True}}_i} + \xi_i$, where $\xi_i \sim N\left(0, \sigma^2 I_d \right)$ and $z^{\text{True}}_i \in \left\{1,2 \right\}$. 

Let $\{z_i\}_{i=1}^K$ be an arbitrary assignment of the samples to the two clusters; we do not assume, for the time being, that it is the correct assignment. 
We denote by $S_k = \left\{i : z_i = k \right\}$ the set of points assigned to cluster $k$. 
We denote by $s_k = \left| S_k \right|$ the size of each cluster.
For convenience, we assume that there are at least two samples assigned to each cluster, so that $s_k > 1$.

W.L.O.G., let us consider an arbitrary sample $x_j$ assigned to a cluster we will denote by $C$ (``current''). 
For this point to be reassigned to the other cluster in the next iteration, it must be closer to the center of the other cluster, which will be denoted by $T$, than to its current center $\mu_C$. The choice of notation $T$ for ``True'' will become apparent in our proof technique, where the ``true cluster`` represents the ``worst case`` for our analysis, but we do not make any assumption at this time.

One may be inclined to analyze the distances between a sample and a cluster composed entirely of samples from the sample's class and the distance between the sample and the center of another cluster composed entirely of samples from the other class.
This may seem to be the ``easiest'' case in some sense.
This analysis is presented in Appendix \ref{sec:warmup} as a ``warmup''.
Our ``worst case'' in the analysis below is closely related to this example, but accounts for a crucial detail: 
the sample $x_j$ is used to compute the center $\mu_C$.

The following lemmas state that the squared distance to cluster $T$ is larger (in the first-order stochastic dominance sense) than $a_T \chi^2_d$ (defined below) and the squared distance to cluster $C$ is smaller (in the same sense) than $a_C \chi^2_d$.
We refer to $a_T \chi^2_d$ and $a_C \chi^2_d$ as the ``worst case''.
The proofs are provided in Appendix \ref{sec:proofs}.

\begin{restatable}[]{lemma}{lemdistcaset}\label{lem:dist:caset:dist}
Under the assumption of Section \ref{sec:distances}, 
the distance to the center of the cluster $T$ has first-order stochastic dominance (see definition \ref{def:FSD}) over the following distribution:
\begin{equation}\label{eq:dist:caset:dist}
    \left\| x_j - \mu_T \right\|^2 \succeq_{FSD} \Delta_T^2 \sim a_T \chi^2_d, ~~~~ a_T = \left( 1 + \frac{1}{s_T} \right) \sigma^2
\end{equation}
where $\chi^2_d$ is distributed chi-squared with $d$ degrees of freedom.
\end{restatable}

\begin{restatable}[]{lemma}{lemdistcasec}\label{lem:dist:casec:dist}
Conversely, under the assumption of Section \ref{sec:distances}, the distance to the center of the cluster $C$ is first-order stochastically dominated by the following distribution:
\begin{equation}\label{eq:dist:casec:dist}
    \left\| x_j - \mu_C \right\|^2 \preceq_{FSD} \Delta_C^2 \sim a_C \chi^2_d, ~~~~ a_C = \left( \sigma ^2 \left(1-\frac{1}{s_C}\right)+\frac{2 \tau ^2 \left(s_C-1\right)^2}{s_C^2} \right)
\end{equation}
where $\chi^2_d$ is distributed chi-squared with $d$ degrees of freedom.

\end{restatable}

\subsection{Differences Between Distances to the Two Cluster Centers}

In this section, we analyze the probability that a single sample $x_j$ is reassigned to the other cluster in the next iteration of the k-means algorithm:
when the difference $\left\| x_j - \mu_T \right\|^2 - \left\| x_j - \mu_C \right\|^2$ is positive, 
the point $x_j$ remains in the $C$ cluster for the next iteration of the k-means algorithm;
when the difference is negative, it is reassigned.

\begin{theorem}\label{thm:maindiff}
    Let $d > 0$ be the dimension of the samples, $n$ be the number of samples. Let $k=2$ be the number of clusters.
    Consider an arbitrary partition of $[n]$ into two mutually exclusive nonempty subsets, each of size larger than $1$.
    Consider an arbitrary sample indexed by $j$ in one of these subsets. 
    We denote by $S_C$ the subset where $j \in S_C$ and by $S_T$ the other subset.
    
    Let the noise level $\sigma > 0$ satisfy:
    \begin{equation}\label{eq:maindiff:req0}
        \sigma >\frac{\sqrt{2} \tau  \left(s_C-1\right)}{\sqrt{\frac{s_C^2}{s_T}+s_C}},
    \end{equation}
    where $s_C = \left| S_C \right|$ and $s_T = \left| S_T \right|$ are the sizes of the clusters.
    
    Let $\mu^{\text{True}}_1,\mu^{\text{True}}_2 \in \mathbb{R}^d$ be the true centers of two clusters, 
     i.i.d. and independent of the partition as $\mu^{\text{True}}_k \sim N\left(0, \tau^2 I_d \right)$ with $\tau > 0$.
    Let the samples be  i.i.d. and independent of the partition and centers, as $x_i = \mu^{\text{True}}_{z^{\text{True}}_i} + \xi_i$, where $\xi_i \sim N\left(0, \sigma I_d  \right)$ is the noise and $z^{\text{True}}_i \in \left\{1,2 \right\}$ is the true subset for each sample.
    Then,
    \begin{equation}
        Pr \left( \left\| x_j - \mu_T \right\|^2 - \left\| x_j - \mu_C \right\|^2 < 0 \right) \leq 
        \rho^{d/4}
    \end{equation}

    where
    {\small
    \begin{equation}\label{eq:maindiff:rho}
        \begin{aligned}
        &\rho(\sigma, \tau, s_C, s_T) = \\
        &\frac{4 \sigma ^2 \left(s_C-1\right) s_C^2 s_T \left(s_T+1\right) \left(s_C \left(\sigma ^2+2 \tau ^2\right)-2 \tau ^2\right)}{\left(-s_C \left(\sigma ^2+4 \tau ^2\right) s_T+s_C^2 \left(\sigma ^2+2 \left(\sigma ^2+\tau ^2\right) s_T\right)+2 \tau ^2 s_T\right){}^2}
        \end{aligned}
    \end{equation}
    }
    and
    \begin{equation}
        0 \leq \rho(\sigma, \tau, s_C, s_T) < 1.
    \end{equation}
\end{theorem}

\begin{proof}

    We recall that the distance between the point and the cluster center $T$ stochastically dominates the ``worst case'' which is a scaled chi-squared distribution with $d$ degrees of freedom (Equation (\ref{eq:dist:caset:dist})), and the distance to the cluster center $C$ is stochastically dominated by the ``worst case'' which is a scaled chi-squared distribution with $d$ degrees of freedom (Equation (\ref{eq:dist:casec:dist})). 
    Substituting $b_1=a_T$ (Equation (\ref{eq:dist:caset:dist})) and $b_2=a_C$ (Equation (\ref{eq:dist:casec:dist})) 
    into Lemma \ref{lem:chibound}, we obtain the desired result.
    The requirement of the Lemma that $b_1 > b_2 >0$ is satisfied by the inequality (\ref{eq:maindiff:req0}).
\end{proof}

The special case $s_T = s_C$, which simplifies the expressions and provides some intuition, is presented in Section \ref{sec:specialcase:equalclusters}.
The more general case is discussed in the following sections.

Using Remark \ref{rem:chibound:rho}, we observe that when $a_C < a_T$ we have $\rho < 1$.
The implication is summarized in the following corollary. 
\begin{corollary}
    Let $\varepsilon_1 > 0$ be a small constant and $d>\frac{4 \log\left( 1/\varepsilon_1 \right)}{\log\left( 1/\rho \right)}$.
    Then, under the assumptions of Theorem \ref{thm:maindiff}, we have 
    \begin{equation}
    Pr \left( \left\| x_j - \mu_T \right\|^2 - \left\| x_j - \mu_C \right\|^2 < m \right) < \varepsilon_1.
    \end{equation}
\end{corollary}
In other words, we can obtain any (un)desired low probability that our chosen sample remains in the same cluster in the next iteration.
This applies to a single sample, but in order to argue that we have a fixed point, we must show that all the samples stay in their current cluster; 
Corollary \ref{cor:special:tau1:sTeqsC:union} in the appendix makes this argument for the special case $s_T = s_C$. The more general case is discussed in the following sections.

\subsection{Samples in Typical Partitions}

We consider again the case of two clusters $K=2$ and investigate partitions where the sizes of the two subsets are close to $n/2$. We recall that as $n$ grows large, most partitions produce clusters whose sizes are close to $n/2$ (Fact \ref{fact:hoeffding:binomial} in Section \ref{sec:preliminaries}).

\begin{restatable}[A ``Typical'' Partition]{theorem}{thmmaindifftypical}\label{thm:maindiff:typical}

    Let $q > 1$ be a constant. 
    Let $d > 0$ be the dimension of the samples.
    Let $n>2 \left(q^2+2\right)+2 \sqrt{q^4+4 q^2}$ be the number of samples. 
    Let $k=2$ be the number of clusters.
    Let $\tau = 1$. 

    Let 
    \begin{equation}\label{eq:maindiff:typical:sigma}
        \sigma = \beta \frac{  \left(\sqrt{n} q+n-2\right)}{\sqrt{2} \sqrt{\sqrt{n} q+n}}
    \end{equation} 
    with $\beta > 1$.

    We consider an arbitrary partition of the samples and a random dataset generated independently from the partition.

    {\underline{Partition}}: Let $S_a$ and $S_b$ be two arbitrary non-empty mutually exclusive subsets of the indices of the samples, 
    chosen independently from the values of the data points,
    and assume that their sizes $|S_a|$ and $|S_b|$ both satisfy 
    \begin{equation}\label{eq:maindiff:typical:sizes}
        n/2 - q \sqrt{n/4} < |S_k| < n/2 + q \sqrt{n/4} ~~\text{ for }~k \in \{a,b\}.
    \end{equation}
    Let $j$ be an arbitrary sample in one of these subsets.

    {\underline{Dataset}}: Let $\mu^{\text{True}}_1,\mu^{\text{True}}_2 \in \mathbb{R}^d$ be the true centers of two clusters, 
    sampled i.i.d. from $\mu^{\text{True}}_k \sim N\left(0, \tau^2 I_d \right)$ with $\tau > 0$.
    Let the samples $\{x_i\}_1^n$ be i.i.d.  $x_i = \mu^{\text{True}}_{z^{\text{True}}_i} + \xi_i$, where $\xi_i \sim N\left(0,  \sigma^2 I_d \right)$ is the noise and $z^{\text{True}}_i \in \left\{1,2 \right\}$ is the true subset for each sample.

    Then,
    \begin{equation}\label{eq:maindiff:typical:pr}
        Pr \left( \left\| x_j - \mu_{\bar{z}(j)} \right\|^2 - \left\| x_j - \mu_{z(j)} \right\|^2 < 0 \right) \leq 
        \rho^{d/4},
    \end{equation}
    where $z(j)$ is its current cluster assignment and $\bar{z}(j)$ is the other cluster, 
    and 
    {\small
    \begin{equation}\label{eq:maindiff:typical:rho}
        \rho = \frac{\sigma ^2 \left(\sqrt{n} q+n-2\right) \left(\sqrt{n} q+n\right) \left(\sqrt{n} q+n+2\right) \left(\sqrt{n} \left(\sigma ^2+2\right) \left(\sqrt{n}+q\right)-4\right)}{\left(n \sigma ^2 \left(\sqrt{n}+q\right)^2+\left(\sqrt{n} q+n-2\right)^2\right)^2}.
    \end{equation}
    }

\end{restatable}

The proof of Theorem \ref{thm:maindiff:typical} is provided in Appendix \ref{sec:proofs}.

\begin{remark}[Asymptotics]\label{rem:maindiff:typical:asymptotics}
    The expressions in Theorem \ref{thm:maindiff:typical} become more interpretable for large $n$.
    Squaring the expression in Equation (\ref{eq:maindiff:typical:sigma}) and expanding it in $n$ yields
    \begin{equation}
        \sigma^2 = \frac{\beta ^2 n}{2}   +\frac{\beta ^2 q  \sqrt{n}}{2} -2 \beta ^2+\frac{2 \beta ^2}{n} +O\left(n^{-3/2}\right).
    \end{equation}
    Substituting Equation (\ref{eq:maindiff:typical:sigma}) into Equation (\ref{eq:maindiff:typical:rho}) and expanding it in $n$ yields
    \begin{equation}
        \rho = 1-\frac{4 \left(\beta ^2  -1\right)^2 n^{-2}}{\beta ^4}
        +\frac{8 \left(\beta ^2-1\right)^2 n^{-5/2} q}{\beta ^4}
        +O\left(n^{-3}\right).
    \end{equation}
\end{remark}

 Theorem \ref{thm:maindiff:typical} implies that in the appropriate regime, the probability that an arbitrary sample in a ``typical'' partition would switch over to a different cluster in the next iteration of the k-means algorithm is small and decreases as the dimension $d$ grows:
\begin{corollary}\label{cor:maindiff:typical:sample}
    Let $\varepsilon_s > 0$ be a small constant and $d>\frac{4 \log\left( 1/\varepsilon_s \right)}{\log\left( 1/\rho \right)}$.
    
    Then, under the assumptions of Theorem \ref{thm:maindiff:typical}, we have 
    \begin{equation}
    Pr \left( \left\| x_j - \mu_{\overline{z(j)}} \right\|^2 < \left\| x_j - \mu_{z(j)} \right\|^2 \right) < \varepsilon_s.
    \end{equation}
    where $z(j)$ is its current cluster assignment and $\overline{z(j)}$ is the other cluster,

\end{corollary}

\subsection{A Typical Partition is a Fixed Point of the k-Means Algorithm}

We observe that Theorem \ref{thm:maindiff:typical} applies to any single sample in a ``typical'' partition that satisfies the assumptions. 
Therefore, using the union bound, we proceed to bound the probability that a partition that satisfies the assumptions is not a fixed point
of the k-means algorithm for a dataset sampled from a GMM model as described in Theorem \ref{thm:maindiff:typical}. 
\begin{corollary}\label{cor:maindiff:typical:single}
    Let $\varepsilon_p > 0$ be a small constant and $d>\frac{4 \log\left( n/\varepsilon_p \right)}{\log\left( 1/\rho \right)}$. 
    
    Then, under the assumptions of Theorem \ref{thm:maindiff:typical}, we have

    \begin{equation}
        Pr \left( \exists j: \left\| x_j - \mu_{\bar{z}(j)} \right\|^2 < \left\| x_j - \mu_{z(j)} \right\|^2  \right) < \varepsilon_p.
    \end{equation}
    where $z(j)$ is $j$'s current cluster assignment and $\bar{z}(j)$ is the other cluster,

\end{corollary}

\begin{proof}

    By the union bound (Equation (\ref{fact:union})), we have
    \begin{equation}
        Pr \left( \exists j: \left\| x_j - \mu_{\bar{z}(j)} \right\|^2 < \left\| x_j - \mu_{z(j)} \right\|^2  \right) \leq
        \sum_j Pr \left( \left\| x_j - \mu_{\bar{z}(j)} \right\|^2 < \left\| x_j - \mu_{z(j)} \right\|^2  \right).
    \end{equation}
    Substituting Equation (\ref{eq:maindiff:typical:pr}), we obtain
    \begin{equation}
        \leq n \rho^{d/4} \leq n \rho^{\left( \frac{4 \log\left( n/\varepsilon_p \right)}{\log\left( 1/\rho \right)} \right)} = \varepsilon_p.
    \end{equation}

\end{proof}

\subsection{Almost All Partitions are Fixed Points of the k-Means Algorithm}

As $n$ grows, the size of the vast majority of clusters becomes very concentrated around $n/2$ (Fact \ref{fact:hoeffding:binomial}).
Therefore, our bounds apply to all partitions except for a small $(1-\delta)$ fraction of the partitions, as formulated in the following corollary.

\begin{restatable}[]{corollary}{cormaindifftypicalall}\label{cor:maindiff:typical:all}

    Let $\delta>0$ and $\varepsilon>0$ be small constants. 
    Let $q>\sqrt{2} \sqrt{-\log \left(\frac{\delta }{4}\right)}$.
    Let $n>2 \left(q^2+2\right)+2 \sqrt{q^4+4 q^2}$ be the number of samples. 
    Let 
    \begin{equation}
    d>\frac{4 \left(\log \left(\frac{n}{\epsilon }\right)+n \log (2)\right)}{\log \left(\frac{1}{\rho }\right)}
    \end{equation}
    
    Then, under the assumptions of Theorem \ref{thm:maindiff:typical} on the sampling of the dataset, we have
    \begin{equation}\label{eq:maindiff:typical:all}
    Pr \left( \text{Number of fixed points} <  (1-\delta) \text{Number of partitions} \right) < \varepsilon.
    \end{equation}

\end{restatable}

The proof is presented in Appendix \ref{sec:proofs}. The main idea is similar to the proof of Corollary \ref{cor:maindiff:typical:single}.

\begin{remark}[Asymptotics]
    We note that the value of $d$ required to satisfy the conditions of Corollary \ref{cor:maindiff:typical:all} depends on $n$ also through $\rho$. 
    If we set $\sigma$ to depend on $n$ as in Equation (\ref{eq:maindiff:typical:sigma}), and follow the asymptotic analysis in Remark \ref{rem:maindiff:typical:asymptotics}, we obtain that the conservative threshold $d$ in this crude conservative analysis behaves like:
       $ d_{th} \approx  \frac{\beta ^4 n^3 \log{2} }{ \left(\beta ^2-1\right)^2 }$.
    This is a conservative estimate, and it is interesting to investigate how the phenomena discussed here affect the convergence of the k-means algorithm in more modest dimensions.
\end{remark}

\section{Numerical Results}\label{sec:numerical}

In this section, we present results for numerical experiments demonstrating Theorems \ref{thm:maindiff} and \ref{thm:maindiff:typical} and their corollaries. 
The code was written in Python using JAX \cite{jax2018github}. 
In order to run many experiments in parallel, we used 16 CPU cores and 128GB of RAM. 
However, single instances of the k-means algorithm at the scale presented here require trivial computational resources. The code required to reproduce all the numerical experiments is available at \href{https://github.com/DSilva27/Observation-on-kmeans}{https://github.com/DSilva27/Observation-on-kmeans}.

\subsection{Probability of cluster reassignment after a step of k-means, special case $|S_T| = |S_C|$}\label{sec:numerical_results:no_reassignment:main}

The first experiment examines Theorem \ref{thm:maindiff} in the special case $|S_T| = |S_C|$ (which is computed explicitly in Corollary \ref{cor:special:tau1:sTeqsC} in the appendix). 

Each instance is an independent experiment where the data (centroids and samples) are generated according to the probabilistic model defined in Theorem \ref{thm:maindiff:typical}, with $n=40$.
We consider two cases: in the ``random'' case, $20$ of the samples are chosen at random with equal probabilities and 
are assigned to $S_C$, the remaining samples are assigned to $S_T$; one sample in $S_C$ is selected in random to the the subject $j$ of the experiment.
In the ``worst case'' we generate data with $|S_1|=21$ and $|S_2|=19$, then define clusters $S_T = S_1 \setminus \{j\}$ and $S_C = S_2 \cup \{j\}$ by moving a single sample $j$ from its true cluster to the wrong cluster. In each instance we examine whether the next k-means iteration moves sample $j$ to the other cluster.

We repeat the experiment $10^5$ times for each of several values of $\sigma^2$ and $d$, and plot in Figure \ref{fig:main_theorem_results} the empirical probability that the sample $j$ moves from its original cluster in the first iteration.  The error interval in the plot is Wilson's interval (See Definition \ref{def:Wilson_interval}). The assumptions of the theorem hold for $\sigma^2 > 18.05$; the bound appears to hold and not be tight.

\begin{figure}
    \centering
    \includegraphics[width=1.0\linewidth]{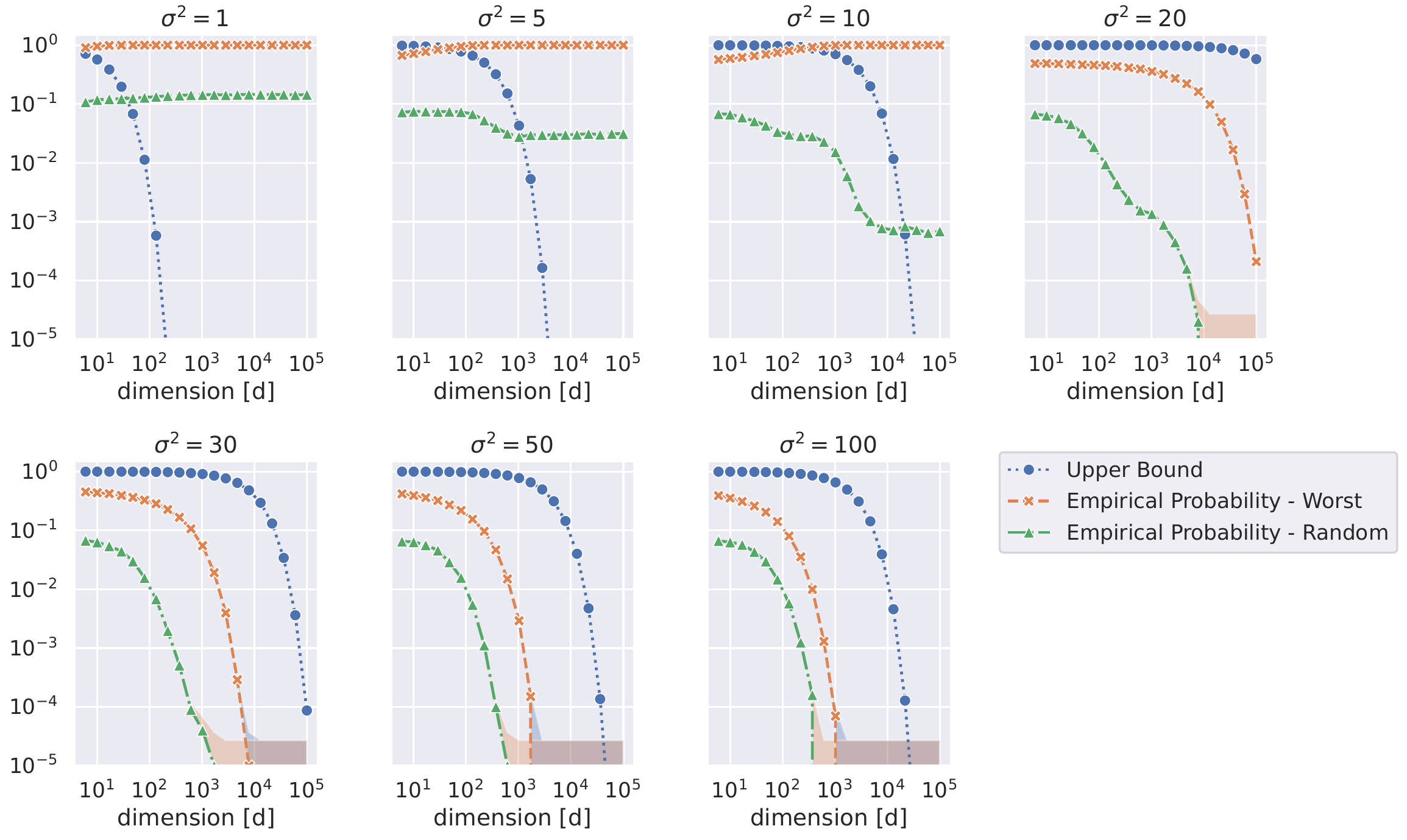}
    \caption{\textbf{Numerical experiments for Theorem \ref{thm:maindiff}}. Ratio of instances where a sample $x_j$ switches clusters after a step of k-means under the probabilistic model defined in Theorem \ref{thm:maindiff}. The ratio is computed for ``random`` (green triangle) and ``worst`` (yellow x) case partitions (see Section \ref{sec:numerical_results:no_reassignment:main}). Sample $j$ is identified as the sample in the ``wrong`` cluster for the ``worst`` case scenario, and a random sample for the ``random`` case. The error interval is Wilson's interval (See Definition \ref{def:Wilson_interval}).
    In addition, we plot the bound provided by Theorem \ref{thm:maindiff} (blue circles). 
    The conditions of Theorem \ref{thm:maindiff} are satisfied by $\sigma^2 > 18.05$.
    }
    \label{fig:main_theorem_results}
\end{figure}

\subsection{Probability of Cluster Reassignment after a Step of k-means: Typical Partitions}\label{sec:numerical_results:no_reassignment:typical}

The following experiment examines Theorem \ref{thm:maindiff:typical} for a typical partition.
Each instance is an independent experiment where the data (centroids and samples) are generated according to the probabilistic model defined in Theorem \ref{thm:maindiff:typical}, with $n=40$.
The assignment $\{ z^{\text{True}}_i \}$ of samples to true centers is i.i.d. with equal probabilities so that the true clusters can vary in size. 
The partition of the samples into subsets for the k-means algorithm is also i.i.d. with equal probabilities and independent of the assignment to true clusters.
In each instance, we examine whether the next k-means iteration moves sample $j$ to the other cluster.

We repeat the experiment $10^4$ times for each of several values of $\sigma^2$ and $d$ and plot in Figure \ref{fig:main_theorem_results} the ratio of instances where sample $j$ moved to the other cluster. The error interval is Wilson's interval (See Definition \ref{def:Wilson_interval}).
We express $\sigma^2$ in terms of $\beta$ as defined in Equation (\ref{eq:maindiff:typical:sigma}).
The threshold for the theorem to hold is $\beta=1$. The bound appears to hold even at slightly lower $\beta$ and it does not appear to be tight.

\begin{figure}
    \centering
    \includegraphics[width=1.0\linewidth]{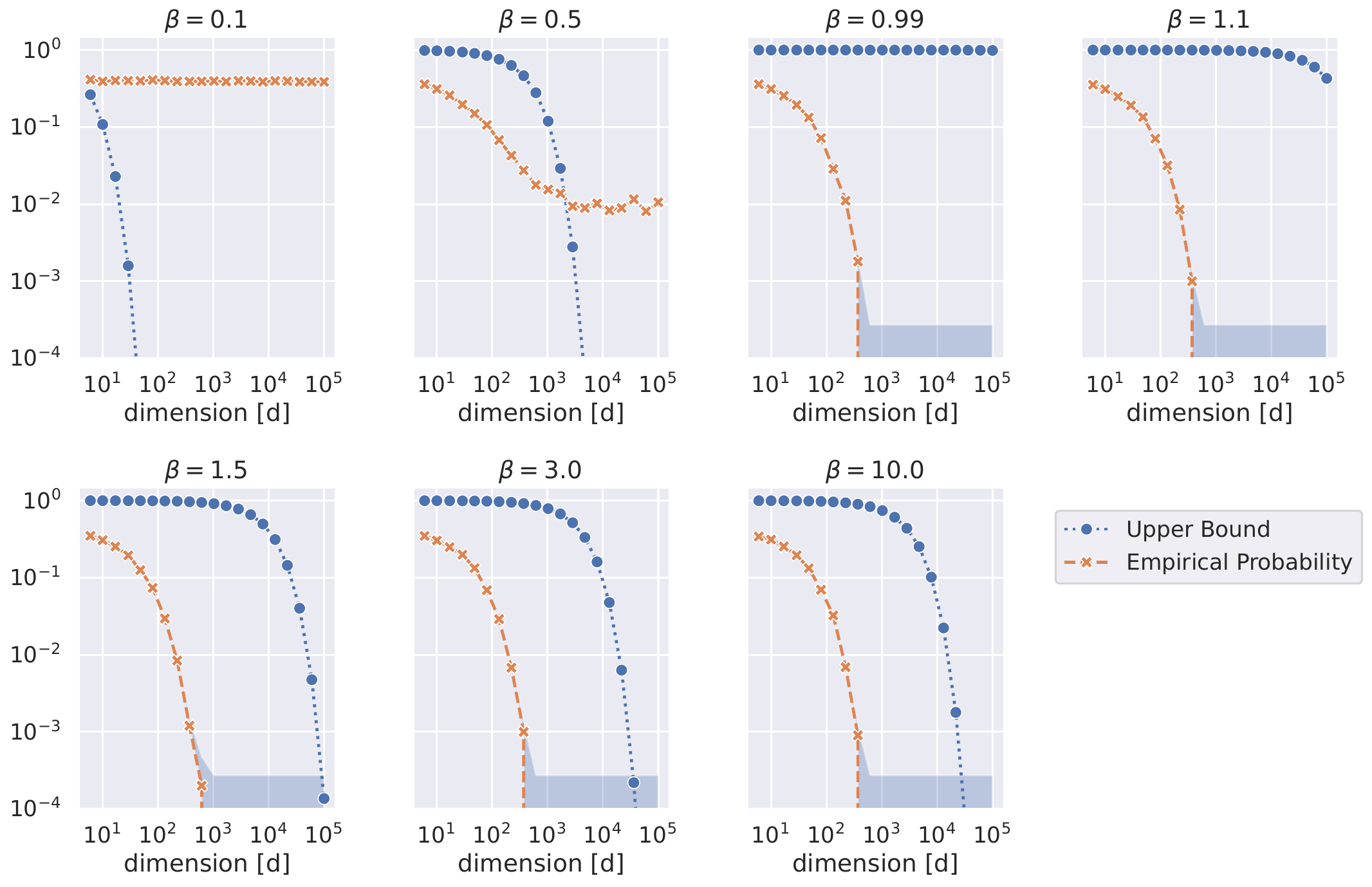}
    \caption{\textbf{Numerical experiments for Theorem \ref{thm:maindiff:typical}.} Ratio of instances where sample $j=0$ switches clusters after a step of k-means under the probabilistic model defined in Theorem \ref{thm:maindiff:typical}. The data are generated as described in Section \ref{sec:numerical_results:no_reassignment:typical}, and assigned to clusters randomly, independently and with equal probability. The error interval is Wilson's interval (See Definition \ref{def:Wilson_interval}).
    The theoretical upper bound, provided by Theorem \ref{thm:maindiff:typical}, is represented as blue circles.
    }
    \label{fig:typical_partition}
\end{figure}

\subsection{k-means in Practice}\label{sec:numerical_results:kmeans_vs_pca}

The following experiment illustrates the relationship between the theoretical findings in this paper and the performance of the k-means algorithm 
using data sampled from the Gaussian Mixture Model (GMM).
Furthermore, it demonstrates how the commonly used PCA dimensionality reduction heuristics (see the discussion in Appendix \ref{sec:settings} and \cite{zha2001spectral,ding2004k}) improve the performance of the algorithm. 
In addition to the standard dimensionality reduction, we introduce a simpler variant (see Appendix \ref{sec:appendix_pca_splitting}), where clustering is performed based on the sign of the first dimension of the data when projected onto the first principal component.
We measure the performance here in terms of the ability of the algorithm to recover the true underlying clustering in terms of the Normalized Mutual Information (NMI) score; see Section \ref{sec:preliminaries:nmi} for details.

\begin{remark}
    We note that we do not require the ``true`` clustering to be a better clustering in terms of the k-means loss; the lowest NMI score does not necessarily correspond to the lowest k-means loss. 
    However, the NMI score is more interpretable across different values of $\sigma$ and $d$, and as the dimension grows it becomes more closely related to the loss.
    Additional results related to the loss are presented in Appendix \ref{sec:additional_results}.
\end{remark}

Each instance of the experiment is generated using the probabilistic model described in Theorem \ref{thm:maindiff}, with $n=40$ data points such that the true clusters have equal size.
In each case, we apply three algorithms: (1) The standard k-means algorithm, (2) k-means on PCA-reduced data (with reduced dimension $d_{\text{PCA}} = 4$), and (3) simple splitting based on the sign of the first principal component (see Appendix \ref{sec:appendix_pca_splitting}). 
We repeat the experiment $100$ times for each value of $d$ and $\sigma^2$.
In the first set of experiments, we initialize the k-means algorithm using randomly selected equal size clusters, with the initial centroids corresponding to the clusters' averages. 
In the second set of experiments, we initialize k-means by designating two randomly selected samples as cluster centers. 
In the third set of experiments, we initialize the k-means algorithm using the popular k-means++ initialization \cite{arthur2006k}.
The results of the experiment are presented in Figure~\ref{fig:experiments:nmi}.
In each set of experiments, we use the same strategy to initialize the standard k-means algorithm and the k-means on PCA reduced data.

Counterintuitively, these experiments demonstrate how the k-means algorithm performs {\em{worse}} as the dimension grows beyond some point, although the additional dimensions can only add information; this added information is exploited by dimensionality reduction methods.
This effect is particularly pronounced in the case of random partition initialization (Figure \ref{fig:experiments:nmi}A), although it is also evident in the case of other initialization methods (Figures \ref{fig:experiments:nmi}B-C). 
Additional results and a discussion are provided in Appendix \ref{sec:additional_results}.

The experiments demonstrate that the k-means algorithm converges to suboptimal fixed points that can be improved upon easily even when some of the information is redacted by dimensionality reduction. The results also highlight the (known) impact of initialization on the quality of the output.
We note that while dimensionality reduction and good initializations are used extensively in applications of the classic k-means algorithm, they are not always immediately generalizable to other applications that motivated this work. 

\begin{figure}
    \centering
    \includegraphics[width=0.7\linewidth]{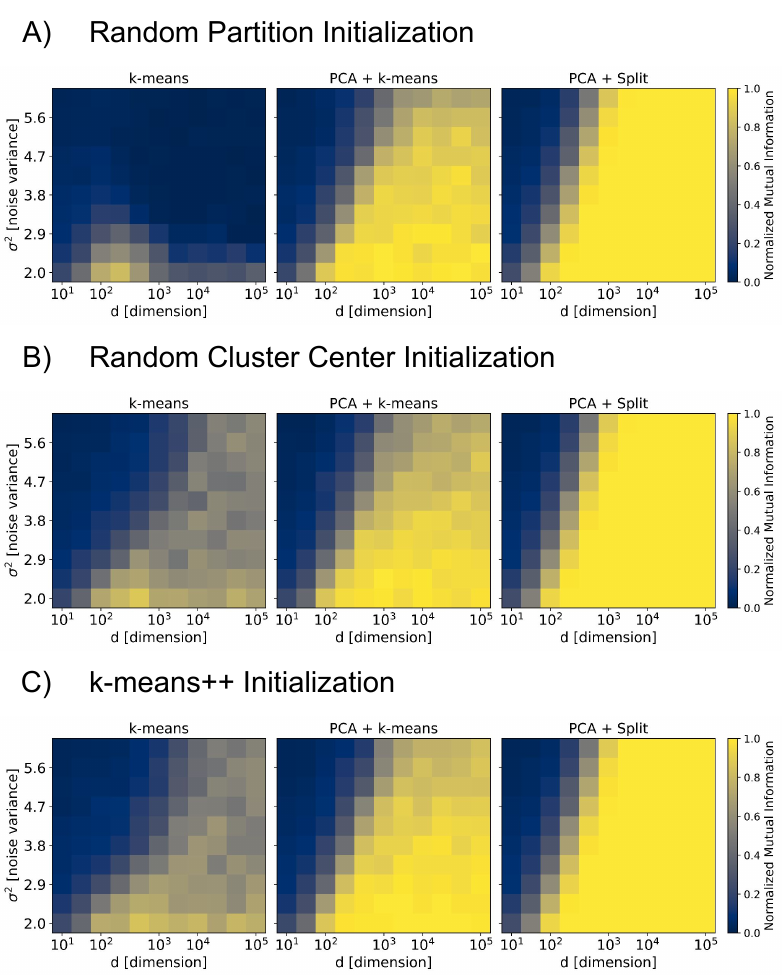}
    \caption{Normalized Mutual Information (NMI) between ground truth assignments and those obtained from the different clustering approaches described in Section \ref{sec:numerical_results:kmeans_vs_pca}. The NMI metric is explained in Section \ref{sec:preliminaries:nmi}. An NMI value of 1 indicates perfect correlation, while a value of 0 signifies no mutual information between two assignments.}
    \label{fig:experiments:nmi}
\end{figure}

\section{Discussion and Conclusions}\label{sec:conclusions}

The analysis of the k-means algorithm in this paper provides one explanation for the observed failure of the algorithm in high-dimensional problems. In sufficiently high noise and dimension, a randomly sampled dataset would, with high probability, have the property that all partitions (with the exception of a small number of atypical partitions) are fixed points of the k-means algorithm, which means that the algorithm is guaranteed to terminate at the initial partition: the partition provided as initialization or a partition computed directly from centers that are provided as the initialization. 
While it is known that the initialization plays an important role in the k-means algorithm and other algorithms, our work points to an extreme case where the initialization is the only thing at play.
The argument in this proof is conservative, and it would be interesting to investigate how the phenomenon influences lower noise and dimensions, and the likelihood that the algorithm gets stuck in sub-optimal fixed points when not all (typical) partitions are fixed points.

In subsequent papers about this work we plan to discuss how alternative algorithms compare to the k-means algorithm. 
Together with these alternatives, we plan to examine the scope of this phenomenon and the performance of alternative algorithms in different settings and in realistic datasets. 
Furthermore, we plan to consider a broader family of problems and algorithms that include EM and more complex settings, leading back to cryo-EM applications that were the motivation for this work. 
As we noted in the appendix, there are many settings where a good initialization can mitigate the effects discussed here. Similarly, in many settings, there are simple spectral alternatives to k-means algorithm that outperform it, and dimensionality reduction preprocessing that improves it. 
We introduced the masked-GMM model to provide a relatively simple setup where the phenomena discussed here apply, but where the spectral methods and initialization methods are not easily applicable.

\begin{ack}

The authors would like to thank 
Tamir Bendory, Amnon Balanov,
Amit Singer, Fred Sigworth, Sheng Xu, Zhou Fan and Yihong Wu for helpful discussions.

The work was supported by NIH/NIGMS (R01GM136780), the Alfred P. Sloan Foundation (FG-2023-20853), AFOSR (FA9550-21-1-0317), the Simons Foundation (1288155), and DARPA/DOD (HR00112490485).
\end{ack}

\bibliography{bibliography}
\bibliographystyle{ieeetr}

\clearpage
\appendix

\section{Model and Problem Formulation}\label{sec:model}

\subsection{Model: The Gaussian Mixture Model (GMM)}\label{sec:GMM}
The Gaussian Mixture Model (GMM) is a probabilistic model often used to model clustering and density estimation problems. 
As its name suggests, the model assumes that the samples are generated from a mixture of Gaussian distributions. 

Let $X = \{x_1, x_2, ..., x_n\}$ be the observed samples, where each $x_i$ is a $d$-dimensional vector. The GMM assumes that each data point $x_i$ is generated from one of $K$ Gaussian components, each component characterized by its mean $\mu_k \in \mathbb{R}^d $ and covariance matrix $\Sigma_k$. 
The component $z_i$ from which $x_i$ was generated is a latent random variable with a categorical distribution with weight $w_k$ for each component.
In summary:
\begin{equation}
    p(z_i = k) = w_k,
\end{equation}
\begin{equation}
    p(x_i, z_i) = p(z_i) p(x_i | z_i) = w_{z_i} \cdot \phi(x_i | \mu_{z_i}, \Sigma_{z_i}),
\end{equation}
where $\phi(x_i | \mu_k, \Sigma_k)$ denotes the probability density function of a Gaussian distribution with mean $\mu_k$ and covariance matrix $\Sigma_k$,
\begin{equation}
    p(x_i) = \sum_{z_i=1}^{K} p(x_i, z_i) = \sum_{k=1}^{K} w_k \cdot \phi(x_i | \mu_k, \Sigma_k).
\end{equation}

{\em In this paper we will restrict our attention to the case where the weights $w_k$ are uniform, i.e., $w_k = 1/K$ for all $k$, and the
covariance matrices are known and isotropic, i.e., $\Sigma_k = \sigma_{\text{noise}}^2 I_d$ for all $k$:}
\begin{equation}\label{eq:gmm:simple:data}
p(x_i  ) = \sum_{k=1}^{K} \frac{1}{K} \cdot \phi(x_i | \mu_k, \sigma_{\text{noise}}^2 I_d)
\end{equation}

\begin{remark}
    One of the difficulties in the statistical problem of characterizing clusters in the GMM in the high-dimensional, finite-data settings, is the large number of parameters that must be estimated. This is a problem primarily due to the covariance matrices, which have $O(d^2)$ parameters each. 
    In the applications that motivated this work, the problem is mitigated using strong assumptions on the structure of the covariance; therefore, this work is not intended to address this issue.  
    Our simple formulation of the problem sidesteps this issue altogether by assuming that the covariance matrices are known and isotropic.
    We hypothesize that the phenomena discussed in this paper are present in more general settings. 
\end{remark}

{\em For the purpose of our analysis, we will consider cluster centers sampled i.i.d. from a Gaussian distribution:}
\begin{equation}\label{eq:gmm:simple:mu}
    \mu_k \sim N(0, \tau^2 I_d).
\end{equation}

\subsection{The k-means Formulation of Recovering Clusters or Cluster Means}\label{sec:kmeans:formulation}

The statistical problem of recovering cluster means in the GMM model involves estimating the unknown parameters for each component based on the observed data. 
In this paper, we restrict our attention to the classic k-means problem of assigning samples to clusters $\left\{ z_i \right\}_{i=1}^n$ and estimating the 
cluster means $\left\{ {\hat{\mu}}_k \right\}_{k=1}^K$,
minimizing the following loss:
\begin{equation}\label{eq:kmeans:loss0}
    \left\{ z_i \right\}_{i=1}^n,  \left\{ {\hat{\mu}}_k \right\}_{k=1}^K \in {\arg\min}_{ \left\{ z_i \right\}_{i=1}^n, \left\{ \mu_k \right\}_{k=1}^{K}} 
    \sum_{i=1}^{n} \min_k \| x_i- \mu_{z_i} \|^2.
\end{equation}

The classic algorithm for solving this problem is the Lloyd k-means algorithm\cite{lloyd1982least}, 
which iteratively assigns each data point to the nearest cluster center and then updates the cluster center to be the mean of the samples assigned to them.
A more detailed description is provided in Section \ref{sec:preliminaries:kmeans}.

\subsection{Settings: Finite-Sample, High-Dimensional, High-Noise Regime}\label{sec:settings}

We will examine the behavior of Lloyd's k-means algorithm in a regime characterized by relatively high noise, where the standard deviation of the noise $\sigma_{\text{noise}}$ often exceeds the signal level $\tau$, and the sample size $n$ is moderate but finite. Specifically, we will consider the scenario where $\frac{\tau^2}{\sigma_{\text{noise}}^2 \cdot n/K} \approx 1$.
Our interest lies in the high-dimensional regime, where the dimensionality $d$ is large and conceptually ``goes to infinity.'' We hypothesize that the phenomena observed in this high-noise, high-dimensional context are indicative of behavior in more general settings.

This scope invites several questions. One of which is {\bf ``Isn't this regime trivial?''} Informally, in simple GMM data, when $d$ is very large, a projection onto the components in the Principal Component Analysis (PCA) of the data easily separates the clusters. In this case, a projection of the data onto the first few principal components makes it easy to separate the clusters; averaging the points in each class yields centers which are about the best we can hope for.
Algorithms using PCA followed by k-means have been proposed, for example, in \cite{zha2001spectral,ding2004k} and are commonly used in practice (see Section \ref{sec:additional_experiments}).
This approach is not immediately applicable to all relevant regimes, but it is applicable to broad relevant regimes. Examples are presented in Section \ref{sec:additional_results:kmeans_in_practice}.

Nevertheless, we study k-means with such high-dimensional regimes in mind because the basic observations in this work impact these algorithms beyond the conservative regime described in this preliminary conservative analysis. 
More importantly, while we discuss the simplest GMM problem, the ideas are applicable to more complicated problems and specifically to applications in cryo-EM where alternative algorithms are not as straightforward. 
In the interest of brevity and to sidestep a tedious discussion of idiosyncrasies of applications such as cryo-EM, we find it useful to consider a modified version of the GMM problem, where each observation is viewed through the lens of observation-specific (known) operator. For concreteness, it is particularly useful to consider the special case of {\bf masked-GMM}, where the observations are viewed through a binary mask. This version can be reformulated as follows:
\begin{equation}\label{eq:maskedGMM}
y_i = A_i \circ x_i,
\end{equation}
where $\circ$ is the Hadamard (element-wise) product, and $A_i$ is a (known) binary vector that masks the $x_i$ generated by the GMM.
The PCA + k-means approach does not immediately apply to masked-GMMs with very sparse masks. Versions of the masked-GMM exist in applications, and they give a glimpse into the difficulties in more complicated applications without delving into details.
Our analysis can be generalized to the masked-GMMs with small modifications; we defer the explicit derivation to future work since the details depend on the parameters of the masks.

A closely related issue is that we do not discuss the initialization (and runs with multiple initializations) in detail in this paper; it is well known that a good initialization can make a significant difference, as demonstrated in Section \ref{sec:additional_results:kmeans_in_practice}. 
Again, the question of initialization is more complicated in more complex setups such as the masked-GMM and a larger number of classes. 
Furthermore, a possible conclusion from this paper is that the initialization might be {\em{the only thing}} at work in some cases, since essentially every partition is a fixed state of the algorithms in some regimes.

Another question about this setting is {\bf ``Is this useful?''} When the number of samples is finite and the noise is high, the cluster centers are a very noisy version of the true cluster centers; sufficiently noisy that one could ask if they are interesting at all.
However, the clustering problem is still valid and in many relevant settings, the clusters are well separated and can be recovered with high probability, even if not using the vanilla k-means algorithm; we touch on this in experiments in Section \ref{sec:additional_results:kmeans_in_practice}.
Furthermore, in more complex settings, which we will investigate in subsequent phases of this work, the equivalent of the cluster centers themselves provide useful information despite the noisy estimates.

\clearpage

\section{Notation}

\begin{table}[ht!]
\centering
\begin{tabular}{ll}
\hline
\textbf{Symbol} & \textbf{Description} \\ \hline
$d$ & Dimension of the samples \\
$n$ & Number of samples \\
$K$ & Number of clusters \\
$I_d$ & Identity matrix \\
$x_i$ & $i$-th sample in $\mathbb{R}^d$ \\
$\mu_k$ & Mean of the $k$-th cluster \\
$\sigma^2$ & Variance of the noise in the data \\
$\tau^2$ & Variance of the cluster means (for brevity we often set $\tau=1$) \\
$z_i$ & Cluster assignment of the $i$-th sample \\
$S_k$ & Set of samples assigned to the $k$-th cluster \\
$S-j$ & Set of samples excluding the $j$-th sample \\
$s_k$ & Size of the $k$-th cluster ($|S_k|$) \\
$\mathcal{N}(\mu, \Sigma)$ & Multivariate normal distribution with mean $\mu$ and covariance $\Sigma$ \\
$\phi(x | \mu, \Sigma)$    & The probability density of a multivariate normal distribution with mean $\mu$ and covariance $\Sigma$ \\
$\chi^2_d$ & Chi-squared distribution with $d$ degrees of freedom \\
$\mathbb{E}[\cdot]$ & Expectation operator \\
$\|\cdot\|$ & Euclidean norm \\ 
$\succeq_{FSD}$ and $\preceq_{FSD}$ & First-order stochastic dominance and first-order stochastic dominance (Definition \ref{def:FSD})\\
\hline
\end{tabular}
\caption{Notation used in the paper.}
\label{tab:notation}
\end{table}

\clearpage

\section{Preliminaries}\label{sec:preliminaries}

This section provides a broader background and standard results that are used in the paper.

\subsection{Lloyd's k-Means Algorithm}\label{sec:preliminaries:kmeans}

The k-means problem has several equivalent formulations. 
One formulation has been presented in Equation (\ref{eq:kmeans:loss0}).
An equivalent formulation states that the purpose of the algorithm it to find a partition or a set of $K$ clusters $S=\{S_1,S_2, \dots, S_K\}$
that minimize the following expression:
\begin{equation}\label{eq:kmeans:loss1}
     \arg \min_S \sum_{k=1}^K \sum_{j \in S_k} \| x_j - \mu_k \|^2 .
    \end{equation}
where $\mu_k= \frac{1}{|S_K|} \sum_{i \in S_k} x_i$.

The k-means problem is known to be NP-Hard \cite{aloise2009np}. Lloyd's algorithm for k-means \cite{lloyd1982least}, to which we refer as {\em{the k-means algorithm}}, is an iterative algorithm intended to find an approximate solution.
The algorithm is initialized with some centers or clusters (a particularly good initialization is known as the k-means++ algorithm \cite{arthur2006k}). It then alternates between assignment and averaging steps.

In the assignment step, each sample is assigned to the nearest cluster mean. The class assignment $z_{j}^{(t)}$ of the point $x_j$ at iteration $t$ is given by:
\begin{equation}\label{eq:kmeans:e}
z_{j}^{(t)} =  \arg \min_k \| x_j - \mu_k^{(t)} \|^2 .
\end{equation}

In the averaging step, the cluster means for the next iteration are defined as the average of the samples assigned to each cluster:
\begin{equation}\label{eq:kmeans:m}
    \mu_k^{(t+1)} = \frac{1}{|S_K^{(t)}|} \sum_{i \in S_k^{(t)}} x_i,
\end{equation}
where $S_k^{(t)} = \{ j | z_{j}^{(t)} = k \}$ is the set of samples assigned to cluster $k$ at iteration $t$.

\begin{remark}
    In the case where one of the clusters $S_k^{(t)}$ is empty at some step, $\mu_k^{(t+1)}$ is not well defined in Equation (\ref{eq:kmeans:m}). 
    This state is known as a degenerate state, and it is a common practical issue that arises when running the Lloyd k-means algorithm. 
    This state is clearly suboptimal, since (in general) moving any sample point into that empty cluster would decrease the loss in Equation (\ref{eq:kmeans:loss1}); several different strategies have been proposed to address this issue.
     For brevity, we exclude the case of an empty subset from our analysis in this paper.
\end{remark}

\subsection{Standard Results}

The following are standard textbook facts in statistics.

\begin{fact}[Adding Gaussian Variables]
    Let $\xi_1$ and $\xi_2$ be i.i.d. with a normal distribution $\xi_1,\xi_2 \sim N(0, 1)$ and let $a_1, a_2, b_1, b_2 \in \mathbb{R}$. 
    Then,
    \begin{equation}\label{eq:sum:gaussian}
        a_1 + b_1 \xi_1 + a_2 + b_2 \xi_2 \sim N(a_1 + a_2, b_1^2 + b_2^2)
    \end{equation}
\end{fact}

\begin{fact}
    Let $X \sim N(0, \sigma^2 )$ have a normal distribution. Let $-\inf <t< 1/2$.
    Then,
    \begin{equation}
        \mathbb{E}\left( exp\left(t X^2 \right) \right) = \frac{1}{\sqrt{1-2 t}}
    \end{equation}
\end{fact}

\begin{fact}[Special Case of Cochran's Theorem]
Let $X \sim N(0, \sigma^2 I_d)$ be a $d$-dimensional Gaussian random vector with mean zero and covariance matrix $\sigma^2 I_d$, where $I_d$ is the $d \times d$ identity matrix.
Then 
\begin{equation}\label{eq:cochran}
\| X \|^2  \sim \sigma^2 \chi^2_d,
\end{equation}
where $\chi^2_d$ denotes the chi-squared distribution with $d$ degrees of freedom.
\end{fact}

The above facts can be used to compute the moment-generating function of the $\chi^2$ distribution.
\begin{fact}[The Moment Generating Function of $\chi^2_d$]
Let $X \sim a \chi^2_d$. Then $E(X) = d a$ and $Var(X) = 2d a^2$. 
Let $t < 1/2$. Then
\begin{equation}\label{eq:chi2:mgf}
    M_X(t) = \mathbb{E} \left( \exp\left( t X \right) \right)= (1-2t)^{-d/2}
\end{equation}
\end{fact}

\begin{fact}[Markov's Inequality]
Let $X$ be a nonnegative random variable. Then for any $a > 0$,
\begin{equation}\label{eq:markov}
    P(X \geq a) \leq \frac{\mathbb{E}(X)}{a}
\end{equation}
\end{fact}

\begin{fact}[Cauchy-Schwarz Inequality]
Let $X$ and $Y$ be random variables. Then,
\begin{equation}\label{eq:cs}
    \mathbb{E}(XY)^2 \leq \mathbb{E}(X^2) \mathbb{E}(Y^2) 
\end{equation}
\end{fact}

\begin{fact}[Chebyshev's Inequality]
Let $X$ be an integrable random with finite variance $\sigma^2>0$ and a finite mean. Then for any $a > 0$,
\begin{equation}\label{fact:chebyshev}
     P(|X - \mathbb{E}(X)| \geq a \sigma ) \leq \frac{1}{a^2}
\end{equation}
\end{fact}

\begin{fact}[Hoeffding's Inequality]\label{fact:hoeffding}
    Let $X_1, X_2, \dots, X_n$ be i.i.d. random variables with $a_i \leq X_i \leq b_i$ almost surely. Let $S = \sum_{i=1}^n X_i$.
    Then for any $t>0$,
    \begin{equation}\label{eq:hoeffding}
        P(|S-\mathbb{E}(S)| \geq t) \leq 2 \exp\left( -\frac{2 t^2}{\sum_{i=1}^n (b_i-a_i)^2} \right)
    \end{equation}
\end{fact}

\begin{fact}[Counting Partitions]\label{fact:binomial}
There are $2^n$ ways to partition $n$ elements into $2$ identified sets.
The fraction of partitions with exactly $S$ elements in the first set (and $n-S$ in the other) is $\binom{n}{S}/2^n$, which coincides with the binomial distribution with $p=1/2$. 
The binomial distribution with $p=1/2$ has a mean of $n/2$ and a variance of $n/4$. 
\end{fact}

\begin{fact}[Counting Typical Partitions]\label{fact:hoeffding:binomial}
    The fractions of the partitions of $n$ into $2$ identified sets with exactly $S$ elements in the first set (and $n-S$ in the other) is bounded by Hoeffding's inequality (Fact \ref{fact:hoeffding}) applied to the binomial distribution:
    \begin{equation}\label{eq:hoeffding:binomial}
        P(|S- n/2 | \geq q \sqrt{n}/2 ) \leq 2 \exp\left( -\frac{q^2}{2} \right)
    \end{equation}
    {\em More informally, for a large $n$, almost all the partitions of $n$ into $2$ identified sets have about $n/2$ elements in each set.}
\end{fact}

\begin{fact}[Union Bound]\label{fact:union}
    Let $A_1, A_2, \dots, A_n$ be events. Then,
    \begin{equation}\label{eq:union}
        P\left( \bigcup_{i=1}^n A_i \right) \leq \sum_{i=1}^n P(A_i)
    \end{equation}
\end{fact}

\begin{definition}[First-Order Stochastic Dominance (FSD)]\label{def:FSD}
    Let $X$ and $Y$ be two random variables. We say that $X$ first-order stochastically dominates $Y$, denoted as $X \succeq_{FSD} Y$, if for all $x \in \mathbb{R}$,
    \begin{equation}\label{eq:FSD}
        P(X \geq x) \geq P(Y \geq x).
    \end{equation}
    In other words, the cumulative distribution function (CDF) of $X$ is always smaller than or equal to the CDF of $Y$ at all points: $F_X(x) \leq F_Y(x)$.
    This implies that $X$ is more likely to have larger values than $Y$.
\end{definition}

\begin{definition}[Wilson's Interval for Confidence Interval of Binomial Proportions]\label{def:Wilson_interval}
    The error bars for estimates of proportions in this paper are computed using Wilson's interval. We choose this method for computing confidence intervals as it is robust to cases where the predicted proportion is close to 1 or 0, a case where other methods for computing the confidence interval give a zero-width interval regardless of the number of samples. The confidence interval is defined as:
    \begin{align}
        CI &= (\mathrm{center} - \mathrm{width}, \mathrm{center}+\mathrm{width}) \\
        \mathrm{center} &= \frac{n_s + \frac{1}{2}z_\alpha^2}{n + z_\alpha^2} \label{results:eq:wilson_interval}\\
        \mathrm{width} &= \frac{z_\alpha}{n + z_\alpha^2} \sqrt{\frac{n_s n_f}{n} +  \frac{z_\alpha^2}{4}}~,
    \end{align}
    where $n$ is the number of experiments, with $n_s$ and $n_f$ being the number of successes and failures, respectively. The value $z_\alpha$ is the $1 - \frac{\alpha}{2}$ for a standard normal distribution. In plots that use Wilson's interval (Figures \ref{fig:main_theorem_results}, \ref{fig:typical_partition}, \ref{fig:warmup_results}, and \ref{fig:union_bound_results}) we plot the actual estimated ratio $n_s/n$, and omit Wilson's center.
\end{definition}

\subsection{Normalized Mutual Information}\label{sec:preliminaries:nmi}

Mutual information quantifies the dependence of two random variables. In our context, we apply it to discrete random variables, although a definition for continuous random variables also exists. Let $X, Y$ be two discrete random variables with a joint probability density function $P_{(X, Y)}$. For the discrete case, the mutual information is defined as:
\begin{equation}
    I(X; Y) = \sum_{x \in \mathcal{Y}} \sum_{x\in \mathcal{X}} P_{(X, Y)}(x, y) \log\left(\frac{P_{(X, Y)}(x,y)}{P_X(x) P_Y(y)}\right)~,
    \label{eq:nmi_definition}
\end{equation}
where $P_X$ and $P_Y$ are the marginal probability density functions of $X$ and $Y$, respectively.

To compare the partitions obtained through k-means to the true partitions, we use the Normalized Mutual Information (NMI). We calculate the NMI using the implementation provided by scikit-learn and apply it to the results shown in Figure \ref{fig:experiments:nmi} (see Section \ref{sec:numerical_results:kmeans_vs_pca}). Scikit-learn's implementation normalizes the Mutual Information (Equation (\ref{eq:nmi_definition})) so that its value ranges between 0 and 1, where 1 indicates perfect correlation, and 0 indicates no dependence.

\section{Proofs}\label{sec:proofs}

This section contains the proofs of the theorems, lemmas, and corollaries presented in the main text.
For convenience, we restate the theorems, lemmas, and corollaries before their proofs. This results in some redundancy in text and numbering - some equation numbers might seem to be out of sequence -
but it may make the proofs easier to follow.

\subsection{Proof of Lemma \ref{lem:chibound}}\label{sec:proof:lem:general:bound}

We restate Lemma \ref{lem:chibound} and provide a proof. 

\lemchibound*

\begin{proof}

    \begin{equation}
        \begin{split}
            Pr&\left( Y_1 - Y_2 -m \leq 0 \right)\\
            &=Pr\left( -(Y_1 - Y_2 -m) \geq 0 \right) \\
              &=Pr\left( \exp\left(-t (Y_1 - Y_2 -m)\right) \geq 1 \right) ~~\text{for all } t>0 .
        \end{split}
    \end{equation}
    Using Markov's inequality (Equation (\ref{eq:markov})):
    \begin{equation}
        \begin{split}
            \leq & \mathbb{E}\left( \exp\left(-t (Y_1 - Y_2 -m)\right) \right) \\
            =    & \exp\left( tm \right) \mathbb{E}\left( \exp\left(-t Y_1  \right) \exp\left( t Y_2 \right) \right)   
        \end{split}
    \end{equation}
    Using the Cauchy-Schwarz inequality (Equation (\ref{eq:cs})):
    \begin{equation}
        \begin{split}
            \leq & \exp\left( tm \right) \sqrt{\mathbb{E}\left( \exp\left(-2t Y_1  \right) \right) \mathbb{E}\left( \exp\left( 2t Y_2 \right) \right)}   
        \end{split}
    \end{equation}
    and using Equation (\ref{eq:chibound:Y1Y2}), 
    \begin{equation}
        \begin{split}
            \leq & \exp\left( tm \right) \sqrt{\mathbb{E}\left( \exp\left(-2t b_1 Z_1  \right) \right) \mathbb{E}\left( \exp\left( 2t b_2 Z_2 \right) \right)}   
        \end{split}
    \end{equation}
    Next, using Equation (\ref{eq:chi2:mgf}), if we further assume $-2t b_1 < 1/2$ and $2t b_2 < 1/2$ (which we will show are satisfied at the optimal $t$), we have:
    \begin{equation}\label{eq:chibound:pf:bound}
        \begin{split}
            = & \exp\left( tm \right) \sqrt{ \left( 1 + 4 t b_1 \right)^{-d/2} \left( 1 - 4 t b_2 \right)^{-d/2} }  \\
            = & \exp\left( tm \right) \left( \left( 1 + 4 t b_1 \right) \left( 1 - 4 t b_2 \right) \right)^{-d/4} 
        \end{split}
    \end{equation}

    The expression $\left(4 b_1 t+1\right) \left(1-4 b_2 t\right)$ is maximized at
    \begin{equation} \label{eq:chibound:pf:tmax}
        t_{\text{max}}=\frac{b_1-b_2}{8 b_1 b_2}.
    \end{equation}

    Using simple algebra for the expression at $t_{\text{max}}$, we have,
    \begin{equation}\label{eq:chibound:pf:inner}
        \left( \left( 1 + 4 t_{\text{max}} b_1 \right) \left( 1 - 4 t_{\text{max}} b_2 \right) \right) = \frac{\left(b_1+b_2\right){}^2}{4 b_1 b_2}.
    \end{equation}

    Substituting Equations (\ref{eq:chibound:pf:tmax}) and (\ref{eq:chibound:pf:inner}) into Equation (\ref{eq:chibound:pf:bound}), we obtain the desired result.

    It remains to check that the assumptions for Equation (\ref{eq:chibound:pf:bound}) are satisfied at $t_{\text{max}}$.
    The first assumption $-2t b_1 < 1/2$ holds because $b_1 > b_2 >0$:
    \begin{equation}
        -2t_{\text{max}} b_1 = -\frac{b_1-b_2}{4 b_2} = \frac{b_2-b_1}{4 b_2} = \frac{1}{4} - \frac{b_1}{4 b_2} < 1/2.
    \end{equation}   

    The second assumption $2t b_2 < 1/2$ holds because $b_1 > b_2 >0$:
    \begin{equation}
        2t_{\text{max}} b_2 = \frac{b_1-b_2}{4 b_1}  < \frac{b_1}{4 b_1} < 1/2.
    \end{equation}

\end{proof}

\begin{remark}[Intuition]
    Under the assumptions of Lemma \ref{lem:chibound}, we have the variables $\tilde{Z}_1 \sim b_1 \chi^2_d$ and $\tilde{Z}_2 \sim b_2 \chi^2_d$.
    The expected value of $\tilde{Z}_1$ is $\mathbb{E}\left(\tilde{Z}_1\right) = d \cdot b_1$ and the expected value of $\tilde{Z}_2$ is $\mathbb{E}\left(\tilde{Z}_2\right) =  d \cdot b_2$. Since $b_1 > b_2$, the center of $\tilde{Z}_1$  is greater than that of $\tilde{Z}_2$. As $d$ increases, the distributions are more concentrated around their means, and therefore we expect that the difference $\tilde{Z}_1 - \tilde{Z}_2$
    will be positive with a probability that increases with $d$. The Lemma states that the probability of a negative value decreases like $\left( \frac{\left(b_1+b_2\right){}^2}{4 b_1 b_2} \right)^{-d/4}$. 
    In other words, since $ \left( \frac{\left(b_1+b_2\right){}^2}{4 b_1 b_2} \right)^{-1} <1 $, we can obtain any desired small probability of $\tilde{Z}_1 - \tilde{Z}_1 \leq m$ simply by increasing $d$.
\end{remark}

\subsection{Proofs for Section \ref{sec:distances}}\label{appendix:distances}

This section discusses the distribution of distances between points and cluster centers in the k-means algorithm in the settings described in Section \ref{sec:settings}.
It provides proofs for lemmas \ref{lem:dist:caset:dist} and \ref{lem:dist:casec:dist} in the main text, and gives more intuition for the reason we call the distributions defined in these lemmas ``worst case.''

\subsubsection{The Distribution of Distance to the $T$ Cluster Center (the cluster to which the sample was {\em not} assigned)}

We first consider the distance between $x_j$ and the center of the $T$ cluster, of which the point $j$ is currently {\em not} a member.

The center $\mu_T$ is the average of all the points assigned to it.
\begin{equation}
\begin{split}
    \mu_T = & \frac{1}{s_T} \sum_{i \in S_T} x_i = \frac{1}{s_T} \sum_{i \in S_T} \left( \mu^{\text{True}}_{z^{\text{True}}_i} + \xi_i \right) = \\
          = & \sum_k \frac{ \left| \{ i \in S_t : z^{\text{True}}_i = k \} \right|}{s_T} \mu^{\text{True}}_{k} + \frac{1}{s_T} \sum_{i \in S_T } \xi_i .
    \end{split}
\end{equation}

The difference between the data point $x_j$ and this cluster center is:
\begin{equation}
    \begin{split}
    x_j - \mu_T = &\mu^{\text{True}}_{z^{\text{True}}_j} + \xi_j     - \frac{1}{s_T} \sum_{i \in S_T } \xi_i.\\
     - &\sum_k \frac{ \left| \{ i \in S_T : z^{\text{True}}_i = k \} \right|}{s_t} \mu^{\text{True}}_{k}
    \end{split}
\end{equation}

We will be interested in the distribution of the squared norm of this difference: $\left\| x_j - \mu_T \right\|^2$.

For our conservative bound on this distance, we consider the case where all the points in class $S_T$ are from the same true cluster as $x_j$, i.e.,
$\forall i \in S_T : z^{\text{True}}_i = z^{\text{True}}_j$. 
This is the ``worst case'' scenario in the sense that the distance between the point and $\mu_T$ is ``minimized''. 
More formally, if these are random variables, the cumulative probability of the distance would be the smallest in this case.
Under this assumption, the difference is:
\begin{equation}
    \begin{split}
     x_j - \mu_T &= \mu^{\text{True}}_{z^{\text{True}}_j} - \mu^{\text{True}}_{z^{\text{True}}_j} + \xi_j - \frac{1}{s_T} \sum_{i \in S_T } \xi_i\\
     &= \xi_j - \frac{1}{s_T} \sum_{i \in S_t } \xi_i.
    \end{split}
 \end{equation}

 Using Equation (\ref{eq:sum:gaussian}) on each dimension of the difference, we obtain the (worst case) distribution of the difference:
 \begin{equation}
    x_j - \mu_T \sim N\left( 0, \left( 1 + \frac{1}{s_T} \right) \sigma^2 I_d \right).
\end{equation}
And using Equation (\ref{eq:cochran}), we obtain the distribution of the squared norm of the difference as a scaled chi-squared random variable and Lemma \ref{lem:dist:caset:dist} which is restated below:

\lemdistcaset*

\begin{remark}
    We note that in this ``worst case'' scenario, the ``signal'' component of the difference, corresponding to the true cluster center disappears.
    However, importantly, there are two sources of ``noise'': the noise from the average of the points, with a variance of $\frac{s_T}{s_T^2} \sigma^2 = \frac{1}{s_T} \sigma^2$, 
    and the noise from the point $x_j$ itself, with a variance of $\sigma^2$; these components sum up to $\left( 1 + \frac{1}{s_T} \right) \sigma^2$ in each dimension.
\end{remark}

\begin{remark}
    Compared to the distance to the correct known cluster considered below in Equation (\ref{eq:wmup_l2_dist_correct}),
    we observe that the distance here (Equation (\ref{eq:dist:caset:dist})), which is the ``worst'' in the sense that it is the ``smallest'', is larger than the distance to the correct cluster by a factor of $1+\frac{1}{s_T}$. This factor is due to the noise from the average of the points in the cluster, and it decreases as the cluster size increases.
\end{remark}

\subsubsection{The Distribution of Distances to the $C$ Cluster Center (the cluster to which the sample is currently assigned)}

We now consider the distance between $x_j$ and the center of the $C$ cluster, of which the point $j$ is currently a member.

The difference between the point $x_j$ and this cluster center is:
{\small 
\begin{equation}
    \begin{split}
        x_j - \mu_C &= \mu^{\text{True}}_{z^{\text{True}}_j} + \xi_j - \frac{1}{s_C} \sum_{i \in S_C } \xi_i \\ &- \sum_k \frac{ \left| \{ i \in S_C : z^{\text{True}}_i = k \} \right|}{s_C} \mu^{\text{True}}_{k} 
    \end{split}
\end{equation}
}
However, since $x_j$ is a member of this cluster and participates in the average, we have:
\begin{equation}
    x_j - \mu_C = 
    {\scriptstyle \left( \begin{aligned}
    & \left( 1 - 1/s_C \right) \left( \mu^{\text{True}}_{z^{\text{True}}_j} + \xi_j \right) \\
    &- \sum_k \frac{ \left| \{ i \in S_C - j : z^{\text{True}}_i = k \} \right|}{s_C} \mu^{\text{True}}_{k} \\
    & - \frac{1}{s_C} \sum_{i \in S_C - j } \xi_i     
    \end{aligned}\right)},
\end{equation}
where $S_C - j$ denotes the set of points in $S_C$ excluding $x_j$.

This time, our ``worst case'' scenario is that all the points in class $S_C$ are from a cluster $q$ different from the true cluster $z_j$ of $x_j$, i.e., 
$\forall i \in S_C : z^{\text{True}}_i = q \neq z^{\text{True}}_j$.
This is the worst case in the sense that the cumulative probability of any distance would be the greatest in this case.

In this case, we have
{\small
\begin{equation}
    \begin{split}
    x_j - \mu_C &= \left( 1 - 1/s_C \right) \left( \mu^{\text{True}}_{z^{\text{True}}_j} + \xi_j \right) - \frac{ s_C -1 }{s_C} \mu^{\text{True}}_{q}\\
    &- \frac{1}{s_C} \sum_{i \in S_C - j } \xi_i.
    \end{split}
\end{equation}
}

Using the independence of the variables and Equation (\ref{eq:sum:gaussian}), we obtain the distribution of the difference:
\begin{equation}
    x_j - \mu_C \sim N \left( 0, \left( \sigma ^2 \left(1-\frac{1}{s_C}\right)+\frac{2 \tau ^2 \left(s_C-1\right)^2}{s_C^2} \right) I_d \right).
\end{equation}
Using Equation (\ref{eq:cochran}), we find that the distribution of the squared norm of the difference is a scaled chi-squared random variable and Lemma \ref{lem:dist:casec:dist} which is restated below:

\lemdistcasec*

\begin{remark}
    Compared to the ``worst case'' in the previous section, in this section there is an additional component corresponding to the distance between the cluster centers. 
    This component contains the signal variance $\tau^2$.

    Importantly, since the point $x_j$ is included in the average, the noise from the point $x_j$ itself is partially canceled out, yielding a minus sign where there is a plus sign in the $T$ case.
    This small difference becomes important when the noise is relatively large.
\end{remark}

\begin{remark}
    Compared to the distance to the wrong known cluster center considered below in Equation (\ref{eq:wmup_l2_dist_wrong}),
    we observe that the distance here in Equation (\ref{eq:dist:casec:dist}), which is the ``worst'' in the sense that it is the ``largest'' possible, is {\em smaller} by a factor of $1-\frac{1}{s_C}$ in the noise component and smaller by a factor of $\frac{\left(s_C-1\right)^2}{s_C^2}$ in the ``signal'' component.
    This is due to the fact that the point $x_j$ is included in the average: the noise and the signal from the point $x_j$ itself are partially canceled out.
    {\bf In other words, the current cluster center is biased toward the sample.}
\end{remark}

\subsection{Proofs of Fixed Points}

\thmmaindifftypical*

\begin{proof}
    W.L.O.G. we assume that $S_b$ is the current cluster of the chosen sample $j$, and $S_a$ is the other cluster,
    i.e. $j \in S_b=S_C$ and $S_a=S_T$.

    We recall from the proof of Theorem \ref{thm:maindiff} that $\rho$ is obtained by substituting $b_1=a_T$ (Equation (\ref{eq:dist:caset:dist})) and $b_2=a_C$ (Equation (\ref{eq:dist:casec:dist})) into Equation (\ref{eq:chibound}).

    By differentiating the expressions in equations (\ref{eq:maindiff:req0}) and (\ref{eq:maindiff:rho}) with respect to $s_C$ and $s_T$, we observe that both the threshold $\sigma$ and the bound $\rho$ are increasing with $s_C$ and $s_T$ in the valid range of Theorem \ref{thm:maindiff}.
    
    Substituting the maximum values of $s_C = n/2 + q \sqrt{n/4}$ and $s_T = n/2 + q \sqrt{n/4}$ into 
    equations (\ref{eq:maindiff:req0}) and (\ref{eq:maindiff:rho}) yields the expressions in the current Theorem.

    We note that slightly better expressions can be obtained by maximizing $\rho$ and $\sigma$ within the range of $s_C$ and $s_T$, and not maximizing the values of $s_C$ and $s_T$ separately.

    The conditions in Equations (\ref{eq:maindiff:typical:sizes}) and (\ref{eq:maindiff:typical:sigma}) ensure that the requirements of Lemma \ref{lem:chibound} are satisfied for any $|S_a|$ and $|S_b|$ in the range of the theorem.

\end{proof}

\cormaindifftypicalall*

\begin{proof}
    The subject of the corollary is a random dataset sampled as described in Theorem \ref{thm:maindiff:typical}.

    By Fact \ref{fact:hoeffding:binomial}, the assumption $q>\sqrt{2} \sqrt{-\log \left(\frac{\delta }{4}\right)}$ excludes at most a $(1-\delta)$ fraction of the partitions.
    These are partitions where the number of samples in each cluster is ``atypically'' far from $n/2$.
    We can now focus our attention on the ``typical'' partitions, which are the vast majority of the partitions.

    By the union bound (Equation (\ref{fact:union})), we have
    \begin{equation}
        Pr \left( \exists \text{Typical partition } P \text{ that is not a fixed point}  \right) \leq
        \sum_{P \in \text{typical partitions}} Pr \left( P \text{ is a fixed point}  \right)
    \end{equation}

    By fact \ref{fact:binomial} There are at most $2^n$ possible partitions of the samples (and therefore at most $2^n$ typical partitions):
    \begin{equation}
       \leq 2^n \min_{P \in \text{typical partitions}} Pr \left( P \text{ is a fixed point}  \right)
    \end{equation}
    Using an argument similar to that in Corollary \ref{cor:maindiff:typical:single}, this expression is bounded by $\varepsilon$.

    Therefore, for a randomly sampled dataset, the probability that all the typical partitions would be a fixed point of the k-means algorithm is at least $\varepsilon$.
    Recalling that we excluded at most a $(1-\delta)$ fraction of the partitions, we obtain Equation (\ref{eq:maindiff:typical:all}).

\end{proof}

\section{Additional Analysis}\label{sec:additional}

\subsection{Warmup: Assignment to the Correct Cluster when the Centers are Known}\label{sec:warmup}

A key component in our analysis is considering the probability that a sample switches from one cluster to another in a Lloyd k-means algorithm iteration (see Section \ref{sec:kmeans:formulation}). 
It is compelling to begin this discussion with the simple question of the probability that a sample is assigned to the correct cluster in the assignment case of the algorithm {\em if the correct cluster centers are known}. This seems to be a reasonable approximation of the problem in a relatively easy case, where the algorithm is close to the correct clustering and the problem is easy enough that the correct clustering is also the optimal clustering that minimizes the loss in Equation (\ref{eq:kmeans:loss0}).
This case also coincides in the limit of a very large number of samples $n$ with our ``worst case'' analysis, where we have the correct partition, except for a single sample (or a small number of samples) in the wrong cluster.
The analysis in the remainder of the paper shows that this intuition is misleading in our setting, but this analysis is a useful warm-up.

Let the two true cluster centers be i.i.d. $\mu_T, \mu_W  \in \mathbb{R}^d \sim N(0, \tau^2 I_d)$, and let the sample $x$ be $x = \mu_T + \xi$, where $\xi \sim N(0, \sigma^2 I_d)$ is independent of the centers. We will analyze the probability that the sample $x$ is closer to the ``wrong center'' $\mu_W$ than it is to the ``true center'' $\mu_T$ from which it was generated.
We begin by analyzing the distribution of distance from the sample to each cluster center.

The difference between the correct cluster center $\mu_T$ and the sample $x$ is:
\begin{equation}
    x - \mu_T = \xi \sim \mathcal{N}(0, \sigma^2 I_d)~.
\end{equation}
Using Equation (\ref{eq:cochran}), we obtain that the distribution of the squared norm of the difference is:
\begin{equation}
    \left\|x - \mu_T\right\|^2  \sim \sigma^2 \chi^2_d~.
    \label{eq:wmup_l2_dist_correct}
\end{equation} 

The difference between the wrong cluster center $\mu_W$ and the sample $x$ is:
\begin{equation}
    x - \mu_W = \mu_T + \xi - \mu_W = \xi + (\mu_T - \mu_W)~.
\end{equation}
Using the independence of the variables and Equation (\ref{eq:sum:gaussian}) we obtain the distribution of the difference:
\begin{equation}
    x - \mu_W\sim \mathcal{N}(0, (\sigma^2 + 2 \tau^2)I_d)~.
\end{equation}
Using Equation (\ref{eq:cochran}), the squared norm of the difference is:
\begin{equation}
    \left\|x - \mu_W\right\|^2  \sim (2\tau^2 + \sigma^2) \chi^2_d~.
    \label{eq:wmup_l2_dist_wrong}
\end{equation}

We now proceed to analyze the probability that the sample is assigned to the wrong cluster center.
\begin{theorem}\label{thm:wmup}
    Let $d > 0$ be the dimension of the samples and the true cluster centers. Let the true cluster centers be i.i.d. $\mu_T, \mu_W \in \Rd \sim N(0, \tau^2 I_d)$, and a sample $x$ be $x = \mu_T + \xi$, with $\xi \sim N(0, \sigma^2Id)$.
    
    Then, 
    \begin{equation}
        \Pr\left(\left\|x - \mu_W\right\|^2 - \left\|x - \mu_T\right\|^2 \leq 0 \right) \leq \rho^{d/4}\\
        \label{eq:wmup_theorem_prob}
    \end{equation}
    where
    \begin{align}
        \rho(\sigma, \tau) &= \frac{1 + \frac{2\tau^2}{\sigma^2}}{\left(1 + \frac{\tau^2}{\sigma^2}\right)^2}\\
    \end{align}
\end{theorem}

\begin{proof}

We note that $\sigma^2 +2\tau^2 > \sigma^2$ for $\tau^2, \sigma^2 > 0$.
Substituting equations (\ref{eq:wmup_l2_dist_correct}) and (\ref{eq:wmup_l2_dist_wrong}) into Lemma \ref{lem:chibound} yields
\begin{equation}
    Pr(\|x - \mu_2\|^2 - \|x - \mu_1\|^2 \leq 0) \leq \rho^{d/4}, 
\end{equation}
where
\begin{align}
    \rho &= \frac{4 \sigma^2 (\sigma^2 + 2\tau^2)}{(\sigma^2 + \sigma^2 + 2 \tau^2)^2}\\
    &= \frac{4 \sigma^4 \left(1 + 2\dfrac{\tau^2}{\sigma^2} \right)}{4\sigma^4 \left(1 + \dfrac{\tau^2}{\sigma^2} \right)^2}\\
    &= \frac{\left(1 + 2\dfrac{\tau^2}{\sigma^2} \right)}{\left(1 + \dfrac{\tau^2}{\sigma^2} \right)^2}~·
\end{align}
\end{proof}

\begin{remark}
    This result coincides with Theorem \ref{thm:maindiff} and Corollary \ref{cor:special:tau1:sTeqsC} when the size of the cluster goes to infinity (in the ``worst case,'' and where only a negligible number of samples is in the wrong cluster).
\end{remark}

\begin{corollary}
    Let $d > \frac{4 \log(1/\varepsilon_0)}{\log(1/\rho)}$, with $\varepsilon_0$ being a small constant. Then, under the conditions of Theorem \ref{thm:wmup} we have:
    \begin{equation}
        Pr(\|x - \mu_W\|^2 - \|x - \mu_T\|^2 \leq 0) \leq \varepsilon_0. 
    \end{equation}
    Then the probability that the point is closer to the wrong center is less than $\varepsilon_0$.
\end{corollary}

This result indicates that if the true clusters are known, a sample is more likely to be closer to the correct cluster, with a probability that increases as the dimension increases. One might speculate that the same holds for k-means, that is, that the averaged cluster centers are close to the true centers and that a point would rarely be assigned to the wrong cluster. 
However, as Theorem \ref{thm:maindiff} shows, this intuition is misleading and misses important finite sample effects, which are particularly important in high-dimensional settings.

\subsection{Special case: $s_T = s_C$}\label{sec:specialcase:equalclusters}

In this section, we consider a special case of Theorem \ref{thm:maindiff}.
A typical partition of the samples will have roughly $s_T \approx s_C \approx n/2$, 
therefore, we consider here the special case where $s_T = s_C$. 
Our more careful analysis in the main text handles ``typical'' partitions more carefully and formally.
The choice $s_T = s_C$ simplifies the expressions and allows us to obtain a more succinct expression for $\rho$.
The choice $s_T = s_C$ also overcomes the inconvenience of the requirement in Equation (\ref{eq:maindiff:req0}) that depends on both $s_T$ and $s_C$ in an asymmetric way.
We observe that the behavior of the problem depends on $\sigma$ and $\tau$ only through the ratio $\sigma/\tau$; 
for brevity, we set $\tau = 1$.

\begin{corollary}[Special case: $\tau=1$ and $s_T = s_C$]\label{cor:special:tau1:sTeqsC}
    Under the assumptions of Theorem \ref{thm:maindiff}, let $\tau=1$ and $s_T = s_C$.
    Then, the requirement in equation (\ref{eq:maindiff:req0}) becomes
    \begin{equation}
        \sigma >\frac{s_C-1}{\sqrt{s_C}},
    \end{equation}
    and $\rho$ simplifies to
    \begin{equation}\label{eq:special:tau1:sTeqsC:rho}
        \rho =\frac{\sigma ^2 s_C \left(s_C^2-1\right) \left(\left(\sigma ^2+2\right) s_C-2\right)}{\left(s_C \left(\left(\sigma ^2+1\right) s_C-2\right)+1\right){}^2}.
    \end{equation}
\end{corollary}

In the simplified case of the latter corollary, we observe that the theorem applies to any of the samples assigned to either cluster.
Therefore, we can use the union bound to obtain the following result.
\begin{corollary}[Union Bound for all Data Points]\label{cor:special:tau1:sTeqsC:union}
    Under the assumptions of Theorem \ref{thm:maindiff}, let $\tau=1$ and $s_T = s_C$.
    Let $\varepsilon_U > 0$ be a small constant and let 
    \begin{equation}
        d \geq \frac{4 \log\left( \frac{1}{\varepsilon_U/n} \right)}{\log\left( \frac{1}{\rho} \right)},
    \end{equation} 
    where $\rho$ is given by Equation (\ref{eq:special:tau1:sTeqsC:rho}).
    
    Then, the probability of any of the data points in either cluster being closer to the cluster to which it is not currently assigned is less than $\varepsilon_U$.
\end{corollary}

{\bf In other words, as $d$ grows, any arbitrary assignment (subject to the conditions of the theorem, and the equal-size simplifying assumption here) becomes a fixed point of the k-means algorithm with high probability.} Similar expressions can be obtained for clusters of different sizes. In the main text, we proceed to examine ``typical'' partitions.

\section{Additional Methods}\label{sec:additional_experiments}

\subsection{Clustering Based on PCA Sign Splitting}\label{sec:appendix_pca_splitting}

In order to demonstrate that in certain regimes the clustering problem becomes ``trivial,'' we introduce a simplified clustering algorithm based on Principal Component Analysis (PCA); another common modification to k-means using PCA is described in the next section.
In the ``trivial'' clustering algorithm, designed specifically for the simple case of two clusters on which we focus in this paper, we compute the first principal component of the data and partition the points based on the sign of their principal component coefficient.

We demonstrate this simplified algorithm in Section \ref{sec:additional_results:kmeans_in_practice}, where the PCA is implemented using Python's scikit-learn.

\subsection{PCA + k-Means}

In addition to the trivial clustering algorithm mentioned previously, a common approach to improve the performance of k-means clustering is to apply PCA to the data as a preprocessing step before k-means \cite{ding2004k}. Unlike the ``trivial'' clustering approach described in the previous section, this method can be used beyond the two-cluster case which is the focus of this paper. 
We demonstrate this preprocessing step in Section \ref{sec:additional_results:kmeans_in_practice}, where the PCA is implemented using Python's scikit-learn.

\section{Additional Numerical Results}\label{sec:additional_results}

\subsection{Values of $\rho$ for Theorem \ref{thm:maindiff}}\label{additional_results:values_of_rho}

Examples of the values of $\rho$ in Equation (\ref{eq:maindiff:rho}) for different parameters are presented in Figure \ref{fig:rho}. The area in the dashed blue line is the region where the requirements of Theorem \ref{thm:maindiff} are not satisfied. We recall that $\rho$ here represents a conservative bound for the ``worst case'' scenario, and therefore some of the effects discussed in this paper are observed beyond the regions where the theorem applies.

\begin{figure}
    \centering
    \includegraphics[width=0.95\linewidth]{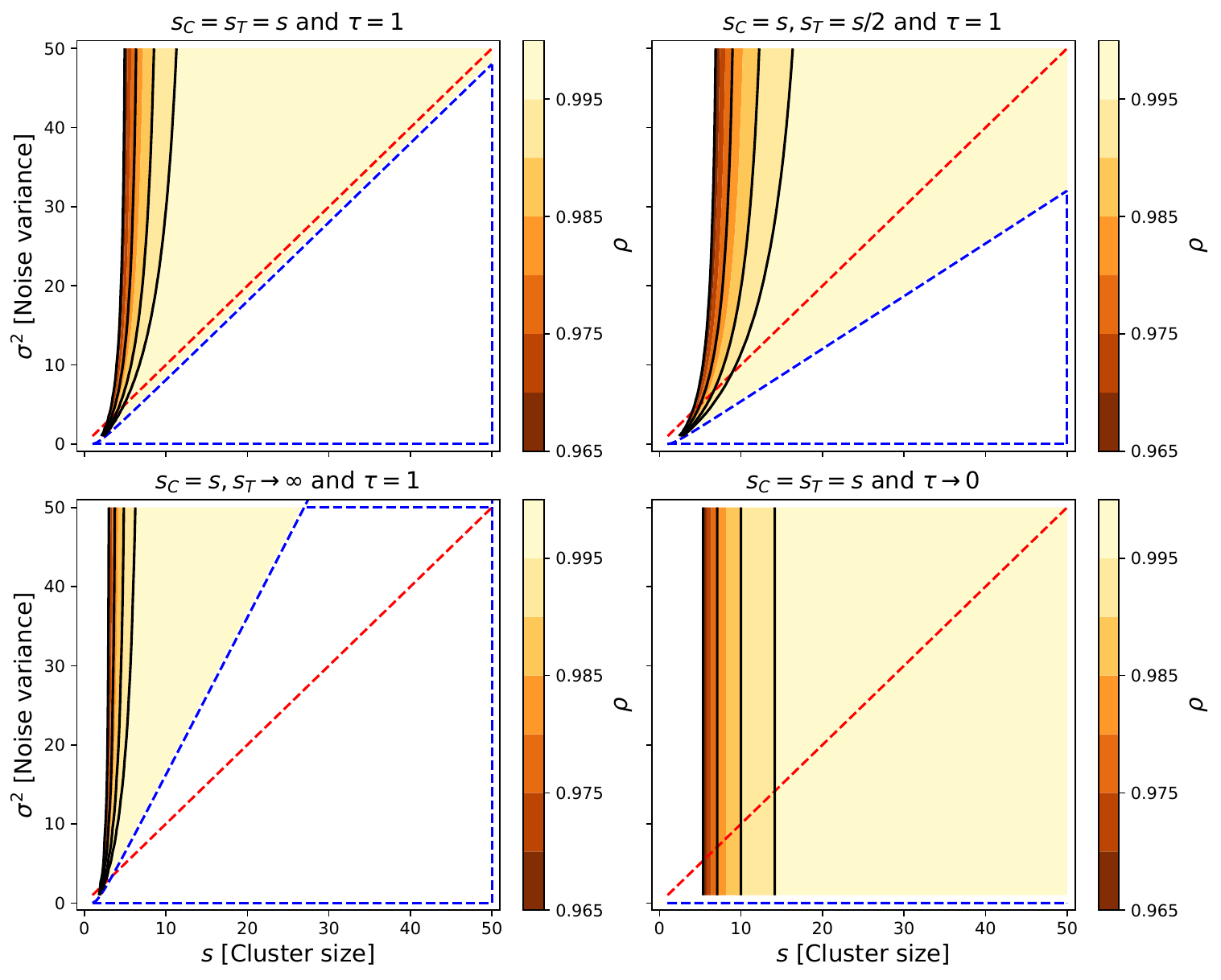}
    \caption{$\rho$ (Equation (\ref{eq:maindiff:rho})) as a function of the noise level $\sigma$ for different cluster sizes $s_C$, $s_T$. The areas bounded by the dashed blue line are regions that do not satisfy the requirements of Theorem \ref{thm:maindiff}. The dashed red line is $\sigma^2 = s_C$.
    }
    \label{fig:rho}
\end{figure}

\subsection{Experiments for Theorem  \ref{thm:wmup} (``Warmup'')}\label{additional_experiments:warmup}

In this section, we present numerical experiments for Theorem \ref{thm:wmup} at different dimensions $d$ and noise levels $\sigma^2$. Each instance is an independent experiment, with data generated according to the probabilistic model defined in Theorem \ref{thm:wmup}. The data consists of two ``clean'' cluster centers $\mu_T\in \Rd$ and $\mu_W \in \Rd$, which are independent and identically distributed
$\mu_T, \mu_W  \sim N(0, I_d)$. 
An independent sample is generated from the $\mu_T$ center: $x = \mu_T + \xi$, where $\xi \sim N(0, \sigma^2 I_d)$. In each experiment, we examine whether the sample $x$ is closer to $\mu_W$ than to $\mu_T$.

We ran $10^5$ experiments for each combination of $d$ and $\sigma^2$.
Figure \ref{fig:warmup_results} illustrates the empirical probability that the sample is closer to the ``wrong center'' $\mu_W$, along with the theoretical upper bounds of Theorem \ref{thm:wmup}.  The error bars were computed using Wilson's interval (see Definition \ref{def:Wilson_interval}).

The experiment is consistent with Theorem \ref{thm:wmup}, and demonstrates how, in this simplified case, as the dimension grows, the probability of incorrect assignment drops to zero in very broad settings. As we observe in the paper, this simplified setup is different from the k-means setup and misses a crucial component, which is how including the sample in the computation of a cluster center biases the cluster center towards that point. This fact explains the difference between the results here and the results in the body of the paper. 

\begin{figure}
    \centering
    \includegraphics[width=1.0\linewidth]{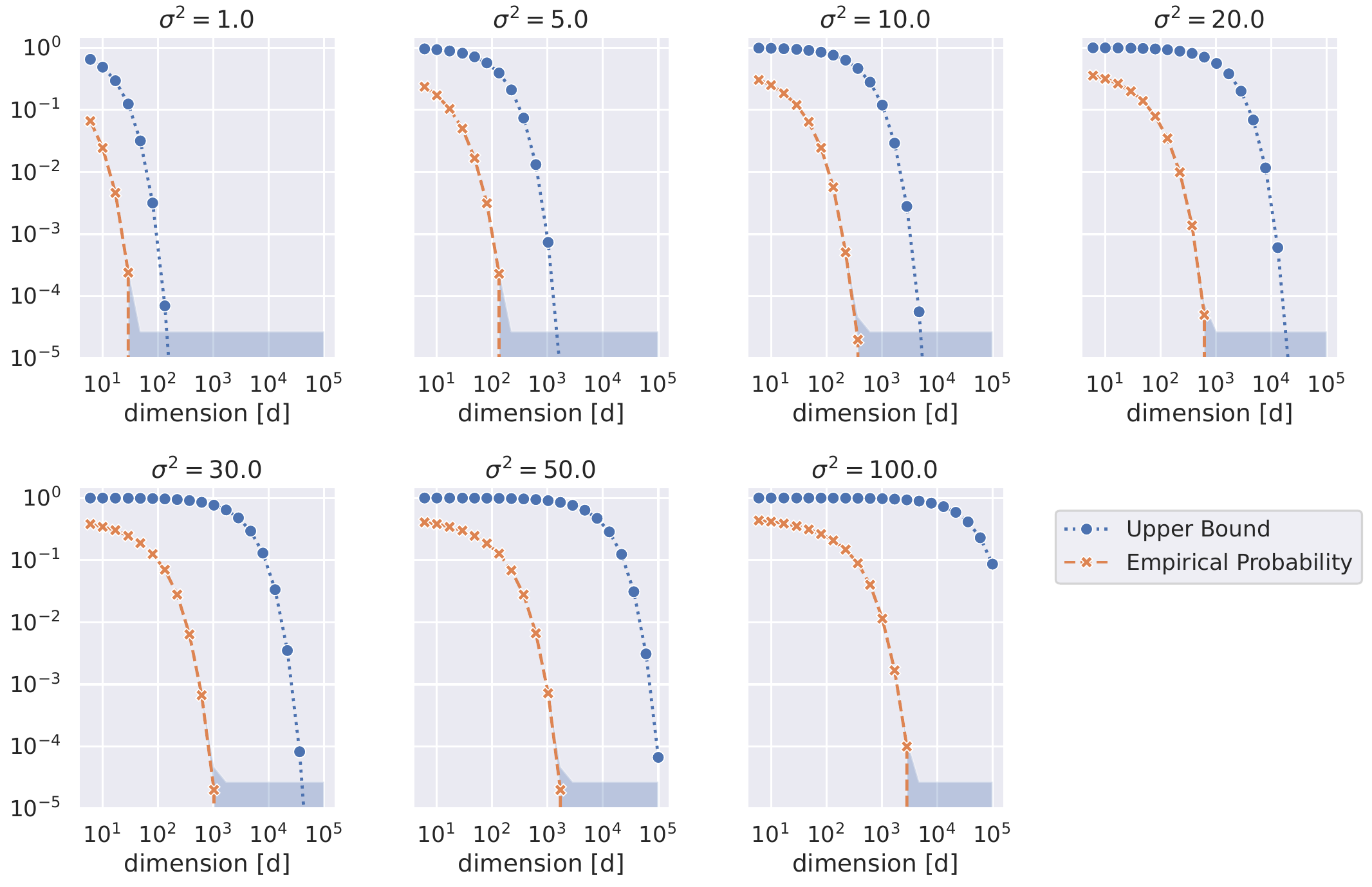}
    \caption{
Numerical results for ``warmup'' Theorem \ref{thm:wmup}. In each experiment, data (centroids and sample) are generated according to the probabilistic model in Theorem \ref{thm:wmup}. We ran a total of $10^5$ experiments for each combination of $d$ and $\sigma^2$, and reported the proportion of times the sample was assigned to the wrong cluster (see Section \ref{additional_experiments:warmup}). Error bars are computed using Wilson's interval (see Definition \ref{def:Wilson_interval}). We observe that, in every case, the proportion of wrong assignments decreases with the dimension. This is consistent with Theorem \ref{thm:wmup}. 
This is different from the behavior of k-means studied in the rest of the paper.
    }
    \label{fig:warmup_results}
\end{figure}

\subsection{Experiments for Corollary \ref{cor:special:tau1:sTeqsC:union}}\label{sec:additional_results:union_bound_cor}

In this section, we discuss the results of numerical experiments based on Corollary \ref{cor:special:tau1:sTeqsC:union}. Each instance is an independent experiment where the data (centroids and samples) are generated as follows: two i.i.d. centroids are randomly generated as $\mu_k^{\text{True}} \in \Rd \sim N(0, I_d)$ for $k=1,2$. Subsequently, 20 i.i.d. samples are generated for each centroid: $x_i = \mu_{z_i^{\text{True}}}^{\text{True}} + \xi_i$, where $\xi_i \sim N(0, \sigma^2Id)$ represents noise, and $z_i^{\text{True}} \in \{1,2\}$ denotes the true subset for each sample. A random partition is created by randomly selecting 20 samples with equal probabilities and assigning them to subset $z_i = 0$. The remaining samples are assigned to subset $z_i = 1$. Using this partition, we perform one step of k-means and examine whether at least one point was reassigned to a different subset.

We generated $10^5$ independent instances of the experiment at each of several combinations of values of $\sigma^2$ and $d$. 
In each case, we compute the proportion of times at least one point swapped clusters, i.e., the empirical probability that the initial partition is not a fixed point of the k-means algorithm. Figure \ref{fig:union_bound_results} shows the estimated proportions, as well as the predicted upper bound. 
Given the parameters used for the experiments, the assumptions in Corollary \ref{cor:special:tau1:sTeqsC:union} hold for $\sigma^2>18.05$, which is similar to the experiments in Figure \ref{fig:main_theorem_results} 
and Figure \ref{fig:union_bound_results}.
The results demonstrate how in these settings and at high dimension almost any partition is a fixed point of the k-means algorithm.

\begin{figure}
    \centering
    \includegraphics[width=1.0\linewidth]{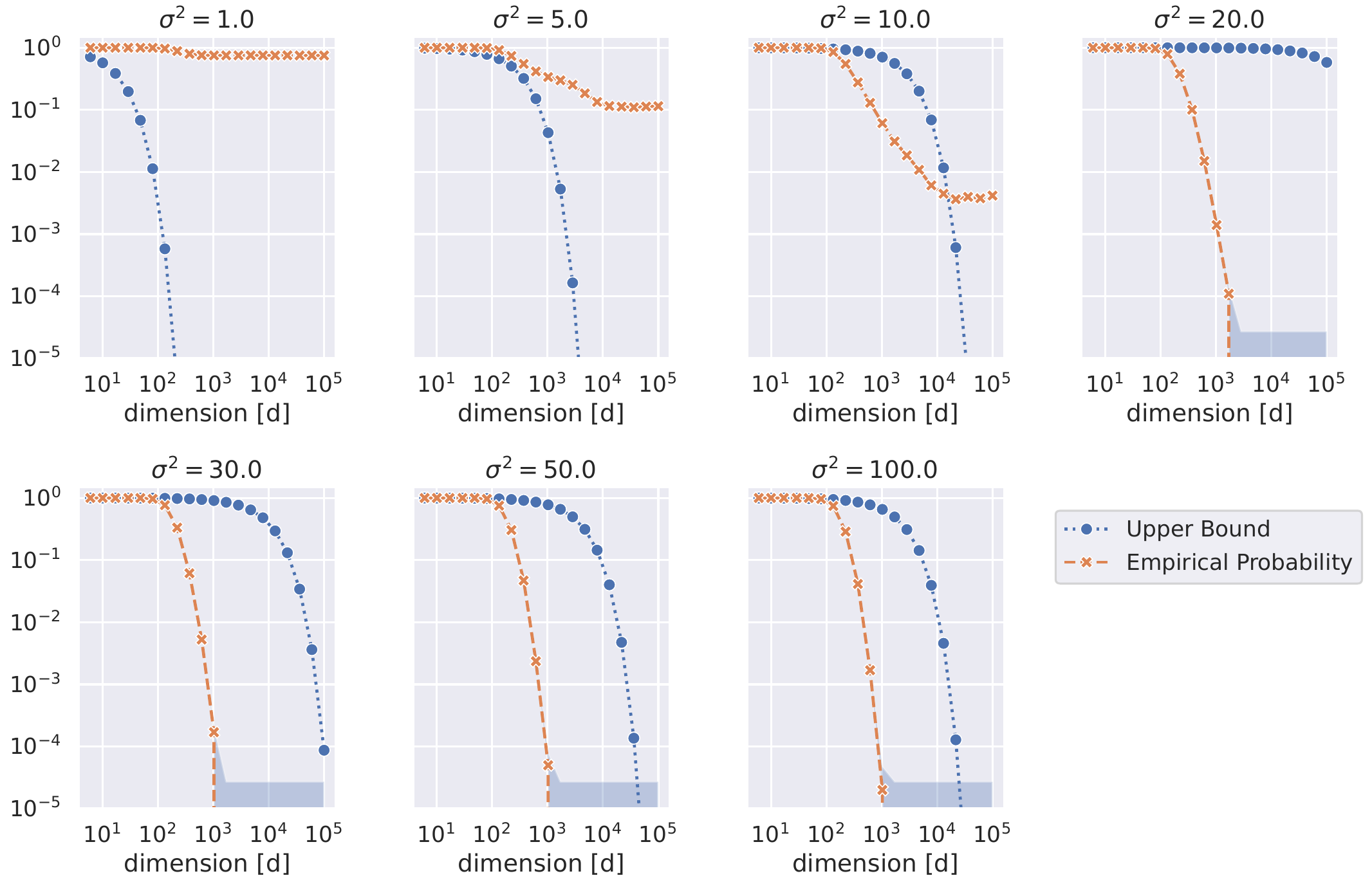}
    \caption{Numerical experiments for Corollary \ref{cor:special:tau1:sTeqsC:union}. For each experiment data (centroids and samples) are generated using the probabilistic model defined in Theorem \ref{thm:maindiff} for the special case of equally populated clusters. In the plots we show the ratio of instances (yellow x) where at least one sample $x_j$ switches clusters after a step of k-means. The error interval is Wilson's interval (See Definition \ref{def:Wilson_interval}).
    In addition, we plot the theoretical upper bound  (Equation (\ref{eq:special:tau1:sTeqsC:rho})) (blue circles). 
    The conditions of Corollary \ref{cor:special:tau1:sTeqsC:union} are satisfied by $\sigma^2 > 18.05$. We observe that as the dimension increases most partitions become a fixed point (no points switch clusters).
    }
    \label{fig:union_bound_results}
\end{figure}

\subsection{Additional Results for k-Means in Practice}\label{sec:additional_results:kmeans_in_practice}

In this section, we revisit the experiments conducted in Section \ref{sec:numerical_results:kmeans_vs_pca} and present the results in terms of the k-means loss (Equation (\ref{eq:kmeans:loss1})). As a benchmark, we use the loss computed for the \textit{ground truth} partition used to generate each dataset.
For each experiment, we record the final loss obtained by each of the clustering methods and each initialization strategy described in Section \ref{sec:numerical_results:kmeans_vs_pca}. 
We note that when dimensionality reduction is used (PCA+ k-means and PCA + split), we use the partition produced by the algorithm to calculate the loss for the {\em original data}, not the reduced-dimension data.

Since the raw loss is difficult to interpret across different parameters, 
we defined a score that examines whether each algorithm produces a partition that is better or worse than the ground truth partition in terms of the loss in Equation (\ref{eq:kmeans:loss1}).
In each instance of the experiment, if the loss calculated for the algorithm's output partition was close to the ground truth loss (up to a relative difference of $10^{-6}$), the score was zero; if the output was better than the ground truth (lower loss), the score was one; if the ground truth was better, the score was negative one. We averaged these scores over all instances generated for each setting to obtain an overall performance measure.

Figure \ref{fig:experiments:loss} illustrates the average scores, comparing (left) k-means, (center) k-means on PCA-reduced data, and (right) clustering by PCA splitting (see Section \ref{sec:additional_experiments}) against the ground truth in terms of the loss calculated for the obtained partition. 
These results are consistent with the NMI scores shown in Figure \ref{fig:experiments:nmi} and demonstrate that at high dimensions, the k-means algorithm fails to minimize Equation (\ref{eq:kmeans:loss1}), and converges to suboptimal fixed points.
As expected, initialization plays an important role, but even relatively good initialization strategies exhibit issues.

\begin{figure}
    \centering
    \includegraphics[width=0.8\linewidth]{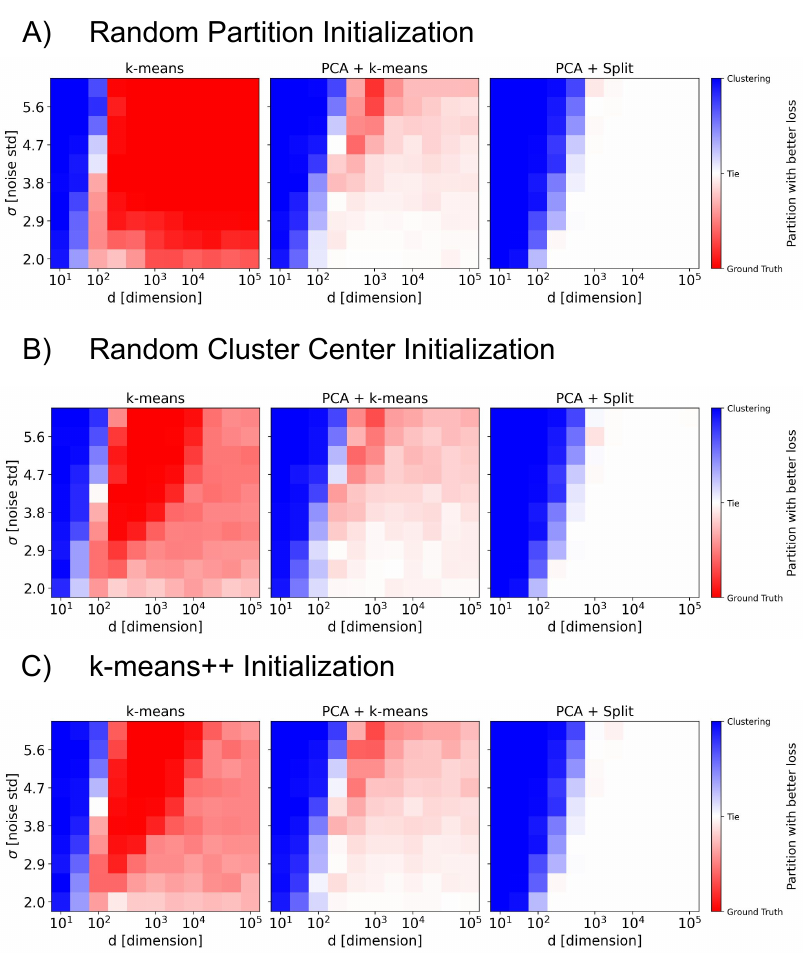}
    \caption{Comparison of the k-means loss obtained with different approaches against the loss obtained with the ground truth partition. The k-means loss is used as a measure of performance, as defined in Equation (\ref{eq:kmeans:loss1}). The figure illustrates the proportion of experiments where each method performed better (red), equal (white), or worse (blue) than the ground truth in terms of the k-means loss. Three different initialization methods are used (A) random partition, (B) random data points assigned as centers, (C) \textit{k-means++}. }
    \label{fig:experiments:loss}
\end{figure}

\end{document}